\documentclass[sigconf, screen]{acmart}

\AtBeginDocument{%
  \providecommand\BibTeX{{%
    \normalfont B\kern-0.5em{\scshape i\kern-0.25em b}\kern-0.8em\TeX}}}

\usepackage{microtype}
\usepackage{graphicx}
\usepackage{booktabs} 

\usepackage{amsmath,amsfonts,bm}

\def\eqref#1{(\ref{#1})}

\def\1{\bm{1}}

\DeclareMathAlphabet{\mathsfit}{\encodingdefault}{\sfdefault}{m}{sl}
\SetMathAlphabet{\mathsfit}{bold}{\encodingdefault}{\sfdefault}{bx}{n}

\usepackage{balance}

\usepackage{url}
\usepackage{subfig}
\usepackage{graphicx}
\usepackage{amsmath}
\usepackage{xspace}
\usepackage{booktabs}
\usepackage{multirow}

\usepackage{mathtools}
\usepackage{amsthm}
\usepackage{paralist}
\usepackage{enumitem}
\usepackage{tabularx}
\usepackage{adjustbox}
\usepackage{graphicx}

\usepackage[ruled,vlined,linesnumbered,lined,boxed,commentsnumbered,algo2e]{algorithm2e}



\newtheorem{problem}{Problem}
\newtheorem{theorem}{Theorem}

\newcommand{\method}{\textsc{Ugnn}\xspace}
\newcommand{\methodadaptive}{\textsc{Ada-Ugnn}\xspace}

\newcommand{\cora}{\textsc{Cora}\xspace}
\newcommand{\citeseer}{\textsc{Citeseer}\xspace}
\newcommand{\pubmed}{\textsc{Pubmed}\xspace}
\newcommand{\blogc}{\textsc{BlogCatalog}\xspace}
\newcommand{\amazoncomp}{\textsc{Amazon-Comp}\xspace}
\newcommand{\amazonphoto}{\textsc{Amazon-Photo}\xspace}
\newcommand{\coauthorcs}{\textsc{Coauthor-CS}\xspace}
\newcommand{\coauthorphys}{\textsc{Coauthor-PH}\xspace}

\SetCommentSty{mycommfont}

\newcommand{\flickr}{\textsc{Flickr}\xspace}
\newcommand{\airusa}{\textsc{Air-USA}\xspace}

\usepackage[normalem]{ulem}

\usepackage{xcolor}

\usepackage{hyperref}
\usepackage{threeparttable}

\begin{document}
\fancyhead{}

\title{A Unified View on Graph Neural Networks\\ as Graph Signal Denoising}

\author{Yao Ma}
\email{yao.ma@njit.edu}
\affiliation{%
  \institution{New Jersey Institute of Technology}
  \city{Newark}
  \state{New Jersey}
  \country{USA}
}

\author{Xiaorui Liu}
\email{xiaorui@msu.edu}
\affiliation{%
  \institution{Michigan State University}
  \city{East Lansing}
  \state{Michigan}
  \country{USA}}

\author{Tong Zhao}
  \email{tzhao2@nd.edu}
\affiliation{%
  \institution{University of Notre Dame}
  \city{Notre Dame}
  \state{Indiana}
  \country{USA}
}

\author{Yozen Liu}
  \email{yliu2@snap.com}
\affiliation{%
  \institution{Snap Inc.}
  \city{Santa Monica}
  \state{California}  
  \country{USA}
}

\author{Jiliang Tang}
\email{tangjili@msu.edu}
\affiliation{%
  \institution{Michigan State University}
  \city{East Lansing}
  \state{Michigan}
  \country{USA}}

\author{Neil Shah}
  \email{nshah@snap.com}
\affiliation{%
  \institution{Snap Inc.}
  \city{Seattle}
  \state{Washington}  
  \country{USA}
}

\renewcommand{\shortauthors}{Trovato and Tobin, et al.}

\begin{abstract}
Graph Neural Networks (GNNs) have risen to prominence in learning representations for graph structured data. A single GNN layer typically consists of a feature transformation and a feature aggregation operation. The former normally uses feed-forward networks to transform features, while the latter aggregates the transformed features over the graph. Numerous recent works have proposed GNN models with different designs in the aggregation operation. In this work, we establish mathematically that the aggregation processes in a group of representative GNN models including GCN, GAT, PPNP, and APPNP can be regarded as (approximately) solving a graph denoising problem with a smoothness assumption. Such a unified view across GNNs not only provides a new perspective to understand a variety of aggregation operations but also enables us to develop a unified graph neural network framework \method. To demonstrate its promising potential, we instantiate a novel GNN model, \methodadaptive, derived from \method, to handle graphs with adaptive smoothness across nodes. Comprehensive experiments show the effectiveness of \methodadaptive. The implementation of ADA-UGNN is available at \url{https://github.com/alge24/ADA-UGNN}.
\end{abstract}



\maketitle

\section{Introduction}
Graph Neural Networks (GNNs) have shown great capacity in learning representations for graph-structured data and thus have facilitated many down-stream tasks such as node classification~\citep{kipf2016semi,velivckovic2017graph,ying2018graph,klicpera2018predict} and graph classification~\citep{defferrard2016convolutional,ying2018hierarchical}. As traditional neural models, a GNN model is usually composed of several stacking GNN layers. Given a graph $\mathcal{G}$ with $N$ nodes, a GNN layer typically contains a feature transformation and a feature aggregation operation as:
\begin{align}
    \text{Feature Transformation: } {\bf X}' = f_{trans}({\bf X}); \nonumber\\
 \text{Feature Aggregation: } {\bf H} = f_{agg}({\bf X}';\mathcal{G}); \label{eq:general_model}
\end{align}
where ${\bf X} \in \mathbb{R}^{N\times d_{in}}$ and ${\bf H}\in \mathbb{R}^{N\times d_{out}}$ denote the input and output features of the GNN layer with $d_{in}$ and $d_{out}$ as the corresponding dimensions, respectively. Similar to traditional neural models, non-linear activation layers are commonly added between consecutive GNN layers. The feature transformation operation $f_{trans}(\cdot)$ transforms the input of ${\bf X}$ to ${\bf X}' \in \mathbb{R}^{N\times d_{out}}$ as its output, and the feature aggregation operation $f_{agg}(\cdot ;\mathcal{G})$ updates node features by aggregating the transformed node features via the graph $\mathcal{G}$.

In general, different GNN models share similar feature transformations (often, a single feed-forward layer), while adopting different designs for the aggregation operation. We raise a natural question -- is there an intrinsic connection among these feature aggregation operations and their assumptions? The significance of a positive answer to this question is two-fold. Firstly, it offers a new perspective to create a uniform understanding on representative aggregation operations. Secondly, it enables us to develop a general GNN framework that not only provides a unified view on multiple existing representative GNN models, but also has the potential to inspire new ones. In this paper, we aim to build the connection among feature aggregation operations of representative GNN models including GCN~\citep{kipf2016semi}, GAT~\citep{velivckovic2017graph}, PPNP and APPNP~\citep{klicpera2018predict}. In particular, we mathematically establish that the aggregation operations in these models can be unified as the process of exactly, and sometimes approximately, addressing a graph signal denoising problem with Laplacian regularization \citep{shuman2013emerging}. This connection suggests that these aggregation operations share a unified goal: to ensure feature smoothness of connected nodes. With this understanding, we propose a general GNN framework, \method, which not only provides a straightforward, unified view for many existing aggregation operations, but also suggests various promising directions to build new aggregation operations suitable for distinct applications and graph properties.  To demonstrate its potential, we build an instance of \method called \methodadaptive, which is suited for handling varying smoothness properties across nodes, and conduct experiments to show its effectiveness.

\section{Representative Graph Neural Networks }\label{sec:graph_and_gnn}
In this section, we introduce notations for graphs and briefly summarize several representative GNN models. A graph can be denoted as $\mathcal{G}=\{\mathcal{V}, \mathcal{E}\}$, where $\mathcal{V}$ and $\mathcal{E}$ are its corresponding node and edge sets. The connections in $\mathcal{G}$ can be represented as an adjacency matrix ${\bf A}\in \mathbb{R}^{N\times N}$, with $N$ the number of nodes in the graph. The Laplacian matrix of the graph $\mathcal{G}$ is denoted as ${\bf L}$. It is defined as ${\bf L} = {\bf D}-{\bf A}$, where ${\bf D}$ is a diagonal degree matrix corresponding to ${\bf A}$. There are also normalized versions of the Laplacian matrix such as ${\bf L}={\bf I} - {\bf D}^{-\frac{1}{2}}{\bf A}{\bf D}^{-\frac{1}{2}}$ or ${\bf L} = {\bf I} - {\bf D}^{-1}{\bf A}$. In this work, we sometimes adopt different Laplacians to establish connections between different GNNs and the graph denoising problem, clarifying in the text. In the following, we generally use ${\bf X}\in \mathbb{R}^{N\times d_{in}}$ and ${\bf H}\in \mathbb{R}^{N\times d_{out}}$ to denote input and output features of GNN layers. ${\bf X}_i$ and ${\bf H}_i$ are used to denote their corresponding $i$-th row, respectively. Next, we describe a few representative GNN models. 
\subsection{Graph Convolutional Networks}
Following \eqref{eq:general_model},  a single GCN layer~\citep{kipf2016semi} can be written as follows:
\begin{align}
\text{Feature Transformation: } {\bf X}' = {\bf X}{\bf W};\nonumber\\
\quad \text{Feature Aggregation: } {\bf H} = \tilde{\bf A}{\bf X}',\label{eq:gcn}
\end{align}
where ${\bf W}\in \mathbb{R}^{d_{in}\times d_{out}}$ is a feature transformation matrix, and $\Tilde{\bf A}$ is a normalized adjacency matrix which includes a self-loop, defined as follows:
\begin{align}
    &\Tilde{\bf A} = \hat{\bf D}^{-\frac{1}{2}} \hat{\bf A} \hat{\bf D}^{-\frac{1}{2}}\label{eq:normlaized_a}, \quad \text{with}\quad  \hat{\bf A} = {\bf A} + {\bf I}, 
 \end{align}
where $\hat{\bf D}$ is the degree matrix corresponding to $\hat{\bf A}$. In practice, multiple GCN layers can be stacked, where each layer takes the output of its previous layer as input. Non-linear activation functions are included between consecutive layers.

\subsection{Graph Attention Networks}
Graph Attention Networks (GAT)~\cite{velivckovic2017graph} adopts the same feature transformation operation as GCN in Eq.~\eqref{eq:gcn}. The feature aggregation operation (written node-wise) for a node $i$ is as: 
\begin{align}
    {\bf H}_i = \sum\limits_{j\in \tilde{\mathcal{N}}(i)}\alpha_{ij} {\bf X}'_j,  \quad \text{with} \quad \alpha_{ij}=\frac{\exp \left( e_{ij}\right)}{\sum\limits_{k \in \tilde{\mathcal{N}}(i)} \exp \left(e_{ik}\right)}. \label{eq:gat_propagation}
\end{align}
where $\tilde{\mathcal{N}}(i)= \mathcal{N}(i)\cup\{ i\}$ denotes $i$'s neighbors (self-inclusive), and ${\bf H}_i$ is the $i$-th row of ${\bf H}$, i.e. the output  features of node $i$. In this aggregation operation, $\alpha_{ij}$ is a learnable attention score to differentiate the importance of distinct nodes in the neighborhood. Specifically, $\alpha_{ij}$ is a normalized form of $e_{ij}$, which is modeled as:
\begin{align}
e_{ij}=\text { LeakyReLU }\left(\left[{\bf X}'_i \| {\bf X}'_j \right] {\mathbf{a}}\right)\label{eq:att_e_ij}
\end{align}
where $[\cdot \|\cdot ]$ denotes the concatenation operation and ${\bf a}\in \mathbb{R}^{2d}$ is a learnable vector. Similar to GCN, a GAT model usually consists of multiple stacked GAT layers.

\subsection{Personalized Propagation of Neural Predictions}
Personalized Propagation of Neural Predictions (PPNP)~\citep{klicpera2018predict} introduces an aggregation operation based on Personalized PageRank (PPR). Specifically, the PPR matrix is defined as $\alpha({\bf I}-(1-\alpha)\tilde{\bf A})^{-1}$, where $\alpha\in (0,1)$ is a hyper-parameter. The $ij$-th element of the PPR matrix specifies the influence of node $i$ on node $j$. The feature transformation operation is modeled as Multi-layer Perception (MLP). The PPNP model can be written in the form of Eq.~\eqref{eq:general_model} as follows:
\begin{align}
    &\text{Feature Transformation: } {\bf X}'_{in} = \text{MLP}({\bf X}); \nonumber\\
    &\text{Feature Aggregation: } {\bf H} =   \alpha({\bf I}-(1-\alpha)\tilde{\bf A})^{-1} {\bf X}'.\label{eq:fa_ppnp}
\end{align}
Unlike GCN and GAT, PPNP only consists of a single feature aggregation layer, but with a potentially deep feature transformation. Since the matrix inverse in Eq.~\eqref{eq:fa_ppnp} is costly, \citet{klicpera2018predict} also introduces a practical, approximated version of PPNP, called APPNP, where the aggregation operation is performed in an iterative way as:
\begin{align}
    {\bf H}^{(k)} = (1-\alpha)\tilde{\bf A}{\bf H}^{(k-1)} + \alpha {\bf X}' \quad k=1,\dots K, \label{eq:appnp_iterative}
\end{align}
where ${\bf H}^{(0)}={\bf X}'$ and ${\bf H}^{(K)}$ is the output of the feature aggregation operation. \citep{klicpera2018predict} shows that ${\bf X}^{(K)}_{out}$ converges to the exact PPNP solution in Eq.~\eqref{eq:fa_ppnp} as $K$ goes to infinity.

\section{GNNs as Graph Signal Denoising}\label{sec:graph_denoising}
In this section, we aim to establish the connections between the introduced GNN models and a graph signal denoising problem with Laplacian regularization. 

\begin{problem}[Graph Signal Denoising]\label{pro:graph_denoising_with_lap}
Given a noisy signal ${\bf S}\in \mathbb{R}^{N\times d}$ on a graph $\mathcal{G}$, the goals is to recover a clean signal ${\bf F}\in \mathbb{R}^{N\times d}$, assumed to be smooth over $\mathcal{G}$, by solving the following optimization problem:
\begin{align}
    \arg\min_{\bf F}~~ \mathcal{L} = \|{\bf F}-{\bf S}\|_F^2 + c \cdot tr({\bf F}^{\top} {\bf L} {\bf F})  .\label{eq:opti_signal}
\end{align}
\end{problem}

The first term guides ${\bf F}$ to be close to ${\bf S}$, while the second term $tr({\bf F}^{\top} {\bf L} {\bf F})$ is the Laplacian regularization which guides ${\bf F}$'s smoothness over $\mathcal{G}$, with $c>0$'s mediation. Assuming that we adopt the unnormalized version of Laplacian matrix with ${\bf L} = {\bf D}-{\bf A}$ (the adjacency matrix ${\bf A}$ is assumed to be binary), the second term in Eq.~\eqref{eq:opti_signal} can be written in an edge-centric way or a node-centric way as:
\begin{align}
 &\text{edge-centric: }  c  \sum\limits_{(i,j)\in \mathcal{E}} \left\lVert{\bf F}_i-{\bf F}_j \right\rVert_2^2;\\
 & \text{node-centric: } \frac{1}{2} c \sum\limits_{i\in \mathcal{V}} \sum\limits_{j\in \tilde{\mathcal{N}}\left(i\right)} \left\lVert{\bf F}_i-{\bf F}_j  \right\rVert_2^2. \label{eq:node-centric-lap}
\end{align}
Clearly, from the edge-centric view, the regularization term measures the \emph{global smoothness} of ${\bf F}$, which is small when connected nodes share similar features. On the other hand, from the node-centric view, we can view the term  $\sum_{j\in \tilde{\mathcal{N}}\left(i\right)} \left\lVert{\bf F}_j-{\bf F}_j \right\rVert_2^2$ as a \emph{local smoothness} measure for node $i$ as it measures the difference between node $i$ and all its neighbors.  The regularization term can then be regarded as a summation of local smoothness over all nodes. Similar formulations can also be derived to other types of Laplacian matrices. 

In the following subsections, we show connections between aggregation operations in various GNN models and Problem \ref{pro:graph_denoising_with_lap}.

\subsection{Connection to PPNP and APPNP}\label{sec:ppnp_appnp}
Our main results linking PPNP and APPNP's to Eq. \eqref{eq:opti_signal} are established in Theorems~\ref{thm:ppnp_denoise} and ~\ref{thm:appnp_denoise}, respectively. 
\begin{theorem}
\label{thm:ppnp_denoise}
When we adopt the normalized Laplacian matrix ${\bf L} = {\bf I}-\tilde{\bf A}$, with $\tilde{\bf A}$ defined in Eq.~\eqref{eq:normlaized_a},
the feature aggregation operation in PPNP (Eq.~\eqref{eq:fa_ppnp}) can be regarded as exactly solving Problem~\ref{pro:graph_denoising_with_lap} with ${\bf X}'$ as the input noisy signal and $c=\frac{1}{\alpha} -1$.
\end{theorem}
\begin{proof}
Note that the objective in Eq.~\eqref{eq:opti_signal} is convex. Hence, the closed-form solution ${\bf F^*}$ of Problem~\eqref{eq:opti_signal} can be obtained by setting its derivative to ${\bf 0}$ as:
\begin{align}
  &\frac{\partial \mathcal{L}}{\partial {\bf F}} = 2( {\bf F} - {\bf S}) + 2c {\bf L} {\bf F} =0 
  \Rightarrow {\bf F^*} = ({\bf I} + c {\bf L})^{-1}{\bf S}   \label{eq:opti_gradient}
\end{align}
Given ${\bf L} = {\bf I} - \Tilde{\bf A}$, 
${\bf F^*}$ can be reformulated as:
\begin{align}
    {\bf F^*} = \left({\bf I} + c {\bf L}\right)^{-1}{\bf S}=   \left({\bf I} + c \left({\bf I} - \tilde{\bf A}\right) \right)^{-1}{\bf S}\nonumber \\=  \frac{1}{1+c} \left({\bf I} - \frac{c}{1+c}\Tilde{\bf A}\right)^{-1} {\bf S} \label{eq:closed-form}
\end{align}
The feature aggregation operation in Eq.~\eqref{eq:fa_ppnp} is equivalent to the closed-form solution in Eq.~\eqref{eq:closed-form} when we set $c=\frac{1}{\alpha} -1$ and ${\bf S} = {\bf X}'$.
This completes the proof.

\end{proof}
\begin{theorem}
\label{thm:appnp_denoise}

When we adopt the normalized Laplacian matrix ${\bf L} = {\bf I}-\tilde{\bf A}$, the feature aggregation operation in APPNP (Eq.~\eqref{eq:appnp_iterative}) approximately solves the graph signal denoising problem~\eqref{eq:opti_signal} by iterative gradient descent with ${\bf X}'$ as the input noisy signal, $c=\frac{1}{\alpha}-1$ and stepsize $b=\frac{1}{2+2c}$.
\end{theorem}
\begin{proof}
To solve the denoising problem in Eq. ~\eqref{eq:opti_signal}, we take iterative gradient method with the stepsize $b$. Specifically, the $k$-th step is as follows:
\begin{align}
    {\bf F}^{(k)}  &\leftarrow {\bf F}^{(k-1)} - b\cdot \left. \frac{\partial \mathcal{L}}{\partial {\bf F}}  \right|_{ {\bf F} = {\bf F}^{(k-1)}} \nonumber\\
    & = (1-2b-2bc){\bf F}^{(k-1)} + 2b{\bf S}  + 2bc\tilde{\bf A}{\bf F}^{(k-1)}
\end{align}
where ${\bf F}^{(0)}={\bf S}$. When we set the stepsize $b$ as $\frac{1}{2+2c}$, we have the following iterative steps:
\begin{align}
     {\bf F}^{(k)}  \leftarrow \frac{1}{1+c} {\bf S} + \frac{c}{1+c} \tilde{\bf A} {\bf F}^{(k-1)}, k=1, \dots K,
\end{align}
which is equivalent to the iterative aggregation operation of the APPNP in Eq.~\eqref{eq:appnp_iterative} with ${\bf S}={\bf X}'$ and $c=\frac{1}{\alpha}-1$ . 

\end{proof}
These two connections provide a new  explanation on the hyper-parameter $\alpha$ in PPNP and APPNP from the graph signal denoising perspective. Specifically, a smaller $\alpha$ indicates a larger $c$, which means the obtained new feature matrix ${\bf H}$ is enforced to be smoother over the graph. 

\subsection{Connection to GCN}
Our main result is established in  Theorem~\ref{thm:gcn_denoise}. 

\begin{theorem}
\label{thm:gcn_denoise}

When we adopt the normalized Laplacian matrix ${\bf L} = {\bf I}-\tilde{\bf A}$, the feature aggregation operation in GCN (Eq.~\eqref{eq:gcn}) can be regarded as solving Problem~\ref{pro:graph_denoising_with_lap} using one-step gradient descent with ${\bf X}'$ as the input noisy signal and stepsize $b=\frac{1}{2c}$. 
 
\end{theorem}
\begin{proof}
The gradient with respect to ${\bf F}$ at ${\bf S}$ is 
   $\left. \frac{\partial \mathcal{L}}{\partial {\bf F}} \right|_{{\bf F} =  {\bf S}} = 2c{\bf L}{\bf S}.$
Hence, one-step gradient descent for the graph signal denoising problem~\eqref{eq:opti_signal} can be described as:
\begin{align}
    {\bf F}   \leftarrow  {\bf S} - b\left. \frac{\partial \mathcal{L}}{\partial {\bf F}} \right|_{{\bf F} =  {\bf X}} &= {\bf S} - 2bc{\bf L S}  \nonumber\\
    &=  (1-2bc ){\bf S}+  2bc\tilde{\bf A}{\bf S}. 
\end{align}
When stepsize $b=\frac{1}{2c}$ and ${\bf S}={\bf X}'$, we have ${\bf F} \leftarrow \tilde{\bf A}{\bf X}'$, which is the same as the aggregation operation of GCN. 
\end{proof}

With this connection, it is easy to verify that a GCN model with multiple GCN layers can be regarded as solving Problem \ref{pro:graph_denoising_with_lap} \emph{multiple times} with different noisy signals as shown in Algorithm~\ref{alg:gcn} (demonstrating for $K$-layer GCN). Specifically, in each layer, the aggregation component aims to solve Problem \ref{pro:graph_denoising_with_lap} with the transformed features as input noisy signal.

\begin{algorithm2e}[t!]
\footnotesize

\SetKwFunction{Union}{Union}\SetKwFunction{FindCompress}{FindCompress}
\SetKwInOut{Input}{input}\SetKwInOut{Output}{output}
\Input{Node Features ${\bf X}$; Adjacency Matrix $\hat{\bf A}$}
\Output{Refined Node Features ${\bf H}$ }
Initialize ${\bf X}^{(0)}\leftarrow {\bf X}, k\leftarrow 1$\;
\While{$1\leq k \leq K$}{ {\bf (Feature Transformation)} ${\bf X}_f^{(k-1)} = {\bf X}^{(k-1)} {\bf W}^{(k-1)}$\;
 {\bf (Feature Aggregation)} Let ${\bf X}_f^{(k-1)}$ be the input noisy signal of Eq.~\eqref{eq:opti_signal}, i.e., ${\bf S}={\bf X}_f$. Solve  Problem~\eqref{eq:opti_signal} via one-step gradient descent as per  Theorem~\ref{thm:gcn_denoise} and denote the solution as ${\bf X}^{(k)}_g$\;
 {\bf (Activation)} ${\bf X}^{(k)}= \sigma({\bf X}^{(k)}_g)$, where $\sigma(\cdot)$ denotes an activation function. \;
 $k \leftarrow k+1$\;
}
${\bf H} = {\bf X^{(K)}}$\;
\Return ${\bf H}$

\caption{$K$-layer GCN As Graph Signal Denoising }\label{alg:gcn}
\end{algorithm2e}
\subsection{Connection to GAT}\label{sec:connection_to_gat}
To establish the connection between graph signal denoising and GAT~\citep{velivckovic2017graph}, in this subsection, we adopt an unnormalized version of the Laplacian, defined based on the adjacency matrix with self-loop $\hat{\bf A}$, i.e. ${\bf L}=\hat{\bf D}-\hat{\bf A}$. Then, the denoising problem in Eq.~\eqref{eq:opti_signal} can be rewritten from a node-centric view as:
\begin{equation}
\resizebox{0.88\linewidth}{!}{
$\arg\min_{\bf F}~~ \mathcal{L} = \sum\limits_{i\in \mathcal{V}} \|{\bf F}_i-{\bf S}_i\|_2^2 +  \frac{1}{2} \sum\limits_{i\in \mathcal{V}} c \sum\limits_{j\in \tilde{\mathcal{N}}(i)} \left\|{\bf F}_i-{\bf F}_j \right\|_2^2,$} \label{eq:opti_signal_node_ecentric_1}
\end{equation}

where $\tilde{\mathcal{N}}(i)= \mathcal{N}(i)\cup\{ i\}$ denotes the neighbors (self-inclusive) of node $i$. In Eq.~\eqref{eq:opti_signal_node_ecentric_1}, the constant $c$ is shared by all nodes, which indicates that the same level of \emph{local smoothness} is enforced to all nodes.  By relaxing this assumption, 
 
instead of a unified $c$ as in Eq.~\eqref{eq:opti_signal_node_ecentric_1}, we can consider a node-dependent $c_i$ for each node $i$:

\begin{equation}
\resizebox{0.88\linewidth}{!}{
$\arg\min_{\bf F}\mathcal{L} = \sum\limits_{i\in \mathcal{V}} \left\lVert{\bf F}_i-{\bf S}_i\right\rVert_2^2 +  \frac{1}{2} \sum\limits_{i\in \mathcal{V}} c_i \sum\limits_{j\in \tilde{\mathcal{N}}\left(i\right)} \left\lVert{\bf F}_i-{\bf F}_j \right\rVert_2^2$.}  \label{eq:adapative_denoising}
\end{equation}

We next show that the aggregation operation in GAT is closely connected to an approximate solution of problem~\eqref{eq:adapative_denoising} with the help of Theorem \ref{thm:gat_denoise}. 

\begin{theorem}\label{thm:gat_denoise}
With adaptive stepsize $b_i =1/\sum\limits_{j\in \tilde{\mathcal{N}}(i)}(c_i+c_j)$ for each node $i$, the process of taking one step of gradient descent from ${\bf S}$ to solve Eq.~\eqref{eq:adapative_denoising} is as follows:
\begin{align}
    {\bf F}_i \leftarrow   \sum\limits_{j\in \tilde{\mathcal{N}}(i)} b_i(c_i+c_j) {\bf S}_j. \label{eq:agg_diff_smooth_att}
\end{align}
\end{theorem}
\begin{proof}
The gradient of Eq.~\eqref{eq:adapative_denoising} with respect to ${\bf F}$ focusing on a node $i$ can be written as:
\begin{align}
    \frac{\partial \mathcal{L}}{\partial {\bf F}_i} = 2\left({\bf F}_i- {\bf S}_i\right) +  \sum\limits_{j\in \tilde{\mathcal{N}}\left(i\right)} \left(c_i+c_j\right) \left({\bf F}_i-{\bf F}_j \right),
\end{align} %
where $c_j$ in the second term appears since $i$ is also in the neighborhood of $j$. Then, the gradient at ${\bf S}$ is
{\small
  $ \left. \frac{\partial \mathcal{L}}{\partial {\bf F}_i}  \right|_{ {\bf F}  = {\bf S}}  = \sum\limits_{j\in \tilde{\mathcal{N}}\left(i\right)} \left(c_i+c_j\right)  \left({\bf S}_i-{\bf S}_j \right).
  $
}
Thus, one step of gradient descent starting from ${\bf S}$ with stepsize $b_i$ is as follows: 

\begin{align}
{\bf F}_i  &\leftarrow {\bf S}_i - b_i \cdot  \left. \frac{\partial \mathcal{L}}{\partial {\bf F}_i} \right|_{{\bf F} = {\bf S}}\nonumber
\\&= \Big( 1 - b_i\sum\limits_{j\in \tilde{\mathcal{N}}\left(i\right)} \left(c_i+c_j\right) \Big) {\bf S}_i +  \sum\limits_{j\in \tilde{\mathcal{N}}\left(i\right)} b_i \left(c_i+c_j\right){\bf S}_j
\label{eq:agg_diff_smooth}
\end{align}
Given $b_i = 1/\sum\limits_{j\in \tilde{\mathcal{N}}(i)} (c_i+c_j)$,  Eq.~\eqref{eq:agg_diff_smooth} can be rewritten as 
\begin{align}
     {\bf F}_i \leftarrow   \sum\limits_{j\in \tilde{\mathcal{N}}(i)} b_i(c_i+c_j) {\bf S}_j, \nonumber
\end{align}
which completes the proof.
\end{proof}

Eq.~\eqref{eq:agg_diff_smooth_att} resembles the aggregation operation of GAT in Eq.~\eqref{eq:gat_propagation} if we treat $b_i(c_i +c_j)$ as the attention score $\alpha_{ij}$. Note that we have {\scriptsize$\sum\limits_{j\in \tilde{\mathcal{N}}(i)}(c_i+c_j)=1/b_i$}, for all $i\in \mathcal{V}$, so, $(c_i+c_j)$ can be regarded as the unnormalized attention score and $b_i$ as the normalization constant. 

We further compare $b_i(c_i+c_j)$ with $\alpha_{ij}$ by investigating the formulation of $e_{ij}$ in Eq.~\eqref{eq:att_e_ij}. Eq.~\eqref{eq:att_e_ij} can be rewritten as:
\begin{align}
e_{ij}=\text { LeakyReLU }({\bf X}'_i  {\bf a}_1 + {\bf X}'_j {\mathbf{a}_2}),\label{eq:att_e_ij_add}
\end{align}
where ${\bf a}_1\in \mathbb{R}^d$ and ${\bf a}_2\in \mathbb{R}^d$ are learnable column vectors, which can be concatenated to form ${\bf a}$ in Eq.~\eqref{eq:att_e_ij}. Comparing $e_{ij}$ with $(c_i+c_j)$, we find that they take a similar form. Specifically, ${\bf X}'_i {\bf a}_1$ and ${\bf X}'_j {\bf a}_2$ can be regarded as the approximations of $c_i$ and $c_j$, respectively. In this way, $e_{ij}$ can be considered as a learnable function estimating $c_i +c_j$. Correspondingly, $b_i(c_i+c_j)$ and $\alpha_{ij}$ are the normalized versions of $c_i+c_j$ and $e_{ij}$, respectively. The difference between $b_i(c_i+c_j)$ and $\alpha_{ij}$ is that the normalization in Eq.~\eqref{eq:agg_diff_smooth_att} for $b_i(c_i+c_j)$ is achieved via summation rather than a softmax as Eq.~\eqref{eq:gat_propagation} for $\alpha_{ij}$. Since GAT makes the $c_i$ and $c_j$ learnable, they also include a non-linear activation in calculating $e_{ij}$. Note that similarly to multi-layer GCN illustrated in Algorithm~\ref{alg:gcn}, multi-layer GAT can be also regarded as solving a series of graph denoising problems in Eq.~\eqref{eq:adapative_denoising}.

\section{\method: A Unified GNN Framework via Graph Signal Denoising}
\label{sec:adaptiveGNN}

In the previous section, we established that the aggregation operations in PPNP, APPNP, GCN and GAT are intimately connected to the graph signal denoising problem with (generalized) Laplacian regularization (Problem \ref{pro:graph_denoising_with_lap}). In particular, all their aggregation operations aim to ensure feature smoothness: either a global smoothness over the graph as in PPNP, APPNP and GCN, or a local smoothness for each node as in GAT. This understanding allows us to develop a unified aggregation operation via the following, more general denoising problem: 

\begin{problem}[Generalized Graph Signal Denoising Problem]
\begin{align}
    \arg\min_{\bf F}~~ \mathcal{L} = \|{\bf F}-{\bf S}\|_F^2 + r(\mathcal{C}, {\bf F}, \mathcal{G}),
\end{align} 
where $r(\mathcal{C}, {\bf F}, \mathcal{G})$ denotes a flexible regularization term to enforce some prior of $\mathbf{F}$ encoded by $\mathcal{G}$.
\label{prob:general_denoising}
\end{problem}
 Note that we overload the notation $\mathcal{C}$ here: it can function as a scalar (like a global constant in GCN), a vector (like node-wise constants in GAT) or even a matrix (edge-wise constants) if we want to give flexibility to each node pair. Different choices of $r(\cdot)$ imply different feature aggregation operations. Besides PPNP, APPNP, GCN and GAT, there are aggregation operations in more GNN models that can be associated with Problem \ref{prob:general_denoising} with different regularization terms such as PairNorm \citep{zhao2019pairnorm} and DropEdge~\citep{rong2019dropedge}.  These two recently proposed enhancements for developing deeper GNN models correspond to the following choices for $r(\mathcal{C}, {\bf F}, \mathcal{G})$:
\begin{align}
    &\text{PairNorm: }   \sum\limits_{(i,j) \in \mathcal{E}} \mathcal{C}_p \cdot \left\| {\bf F}_i - {\bf F}_j\right\|_2^2 -  \sum\limits_{(i,j) \not\in \mathcal{E}} \mathcal{C}_n \cdot \left\| {\bf F}_i - {\bf F}_j\right\|_2^2,  \nonumber\\
    & \text{DropEdge: }  \sum\limits_{(i,j) \in \mathcal{E}} \mathcal{C}_{ij} \cdot \left\| {\bf F}_i - {\bf F}_j\right\|_2^2, \ \text{where} \ \mathcal{C}_{ij}\in \{0,1\}. \nonumber    
\end{align}
For PairNorm, $\mathcal{C}$ consists of $\mathcal{C}_p, \mathcal{C}_n>0$ and the regularization term ensures connected nodes to be similar while disconnected nodes to be dissimilar. For DropEdge, $\mathcal{C}$ is a sparse matrix having the same shape as the adjacency matrix. For each edge $(i,j)$, its corresponding $\mathcal{C}_{ij}$ is sampled from a Bernoulli distribution with mean $1-q$, where $q$ is a pre-defined dropout rate. The above mentioned regularization terms are all related to the Laplacian regularization. Other regularization terms can also be adopted, which may lead to novel designs of GNN layers. For example, if we aim to enforce piece-wise linearity in the clean signal, we can adopt $r(\mathcal{C}, {\bf F},\mathcal{G} ) = \mathcal{C} \cdot \left\| {\bf L} {\bf F} \right\|_1$ designed for trend filtering~\citep{tibshirani2014adaptive,wang2016trend}.

With these discussions, we propose a unified framework (\method) to design GNN layers from the graph signal processing perspective: 1) Design a graph regularization term $r(\mathcal{C}, {\bf F}, \mathcal{G})$ in Problem \ref{prob:general_denoising} according to specific applications; 2) \text{Feature Transformation: } ${\bf X}' = f_{trans}({\bf X})$; and 3) \text{Feature Aggregation: } \text{Solve Problem \ref{prob:general_denoising} with}  ${\bf S} = {\bf X}'$ and the designed $r(\mathcal{C}, {\bf F}, \mathcal{G})$.

To demonstrate the potential of \method, we next introduce a new GNN model, \methodadaptive by instantiating \method with $r(\mathcal{C}, {\bf F}, \mathcal{G})$ enforcing adaptive local smoothness across nodes.  
\section{\methodadaptive: Adaptive Local Smoothing with \method}\label{sec:ada-ugnn}
From the graph signal denoising perspective,  PPNP, APPNP, and GCN enforces global smoothness by penalizing the difference with a constant $\mathcal{C}$ for all nodes. However, real-world graphs may consist of multiple groups of nodes which have different behaviors in connecting to similar neighbors.  For example, Section~\ref{sec:node_classification} shows several graphs with varying distributions of local smoothness (as measured by label homophily): summarily, not all nodes are highly label-homophilic, and some nodes have considerably ``noisier'' neighborhoods than others. Moreover, as suggested by~\citet{wu2019adversarial,jin2020graph}, adversarial attacks on graphs tend to promote such label noise in graphs by connecting nodes from different classes and disconnecting nodes from the same class, rendering resultant graphs with varying local smoothness across nodes. Under these scenarios, a constant $\mathcal{C}$ might not be optimal, suggesting the value of adaptive (i.e. non-constant) smoothness assumptions. As shown in Section~\ref{sec:connection_to_gat}, by viewing GAT's aggregation as a solution to regularized graph signal denoising, GAT can be regarded as adopting an adaptive $\mathcal{C}$ for different nodes, which facilitates adaptive local smoothness. However, in GAT, the graph denoising problem is solved by a single step of gradient descent, which might still be suboptimal. Furthermore, when modeling the local smoothness factor $c_i$ in Eq.~\eqref{eq:agg_diff_smooth_att}, GAT only uses features of node $i$ as input, which may not be optimal since intuitively, understanding $c_i$ as local smoothness, it should be intrinsically related to $i$'s \emph{neighborhood}. In this section, we adapt this notion directly into the \method framework by introducing a new regularization term, and develop a resulting GNN model (\methodadaptive) which aims to enforce adaptive local smoothness to nodes in a different manner to GAT. We then utilize an iterative gradient descent method to approximate the optimal solution for Problem \ref{prob:general_denoising} with the following regularization term:
\begin{align}
     r(\mathcal{C},{\bf F}, \mathcal{G}) =  \frac{1}{2}  \sum\limits_{i\in \mathcal{V}} \mathcal{C}_i \sum\limits_{j\in \tilde{\mathcal{N}}(i)} \left\| \frac{{\bf F}_i}{\sqrt{d_i}}-\frac{{\bf F}_j}{\sqrt{d_j}}. \right\|_2^2 \label{eq:normalized_adaptive}
\end{align}
where $d_i, d_j$ denote the degree of nodes $i$ and $j$ respectively, and $\mathcal{C}_i$ indicates the smoothness factor of node $i$, which is assumed to be a fixed scalar. Note that, the above regularization term can be regarded as a generalized version of the regularization term used in PPNP, APPNP, and GCN. Similar to PPNP and APPNP, \methodadaptive only consists of a single GNN layer. We next describe the feature transformation and  aggregation operations of \methodadaptive, and show how to derive the model via \method.  

\subsection{Feature Transformation}
Similar to PPNP and APPNP, we adopt MLP for the feature transformation. Specifically, for a node classification task, the dimension of the output of the feature transformation ${\bf X}'$ is the number of classes in the graph. 

\subsection{Feature Aggregation}

We use iterative gradient descent to solve Problem \ref{prob:general_denoising} with the regularization term in Eq.~\eqref{eq:normalized_adaptive}. The iterative gradient descent steps are stated in the following theorem.

\begin{theorem}\label{thm:ada-ugnn_denoise}
With adaptive stepsize {\footnotesize $b_i =  1/\left(2+ \sum\limits_{j\in \tilde{\mathcal{N}}(i)} (\mathcal{C}_i+\mathcal{C}_j)/{d_i}\right)$} for each node $i$, the iterative gradient descent steps to solve Problem~\ref{prob:general_denoising} with the regularization term in Eq.~\eqref{eq:normalized_adaptive} is as follows:
\begin{equation}
\resizebox{0.88\linewidth}{!}{
$       {\bf F}^{(k)}_i \leftarrow 2b{\bf S}_i + b_i  \sum\limits_{j\in \tilde{\mathcal{N}}(i)} ({\mathcal{C}_i+\mathcal{C}_i})  \frac{{\bf F}^{(k-1)}_j}{\sqrt{d_i d_j}};\quad  k=1,\dots ,$
}\label{eq:iterative_simplified}
\end{equation}

where {\small${\bf F}^{(0)}_i = {\bf S}_i$}.

\begin{proof}
The gradient of the optimization problem~\ref{prob:general_denoising} with the regularization term in Eq.~\eqref{eq:normalized_adaptive} with respect to ${\bf F}$ (focusing on node $i$) is as follows:
\begin{align}
    \frac{\partial \mathcal{L}}{\partial {\bf F}_i} = 2({\bf F}_i - {\bf S}_i) +  \sum\limits_{v_j\in \tilde{\mathcal{N}}(v_i)} \frac{\mathcal{C}_i + \mathcal{C}_j}{\sqrt{d_i}}  \left(\frac{ {\bf F}_i}{\sqrt{d_i}}-\frac{{\bf F}_j}{\sqrt{d_j}} \right),  \label{eq:gradient}
\end{align}
where $\mathcal{C}_j$ in the second term appears since node $i$ is also in the neighborhood of node $j$.
The iterative gradient descent steps with adaptive stepsize $b_i$ can be formulated as follows:
\begin{align}
  {\bf F}^{(k)}_i \leftarrow {\bf F}^{(k-1)}_i- b_i\cdot \left.\frac{\partial \mathcal{L}}{\partial {\bf F}_i}\right|_{{\bf F}_i={\bf F}^{(k-1)}_i} ; \quad  k=1,\dots \label{eq:iterative_adpative}
\end{align}
With the gradient in Eq.~\eqref{eq:gradient}, the iterative steps in Eq.~\eqref{eq:iterative_adpative} can be rewritten as:
\begin{align}
    {\bf F}^{(k)}_i \leftarrow & (1-2b_i-b_i \sum\limits_{v_j\in \tilde{\mathcal{N}}(v_i)} \frac{\mathcal{C}_i+\mathcal{C}_j}{d_i} ) {\bf F}^{(k-1)}_i + 2b_i{\bf S}_i \nonumber \\
     &+ b_i  \sum\limits_{v_j\in \tilde{\mathcal{N}}(v_i)} ({\mathcal{C}_i+\mathcal{C}_j})  \frac{{\bf F}^{(k-1)}_j}{\sqrt{d_i d_j}};\quad  k=1,\dots  \label{eq:iter}
\end{align}
Given $b_i  =  1/\left(2+ \sum\limits_{v_j\in \tilde{\mathcal{N}}(v_i)} (\mathcal{C}_i+\mathcal{C}_j)/{d_i}\right)$, the iterative steps in Eq.~\eqref{eq:iter} can be re-written as follows:
\begin{align}
       {\bf F}^{(k)}_i \leftarrow 2b_i{\bf S}_i + b_i  \sum\limits_{v_j\in \tilde{\mathcal{N}}(v_i)} ({\mathcal{C}_i+\mathcal{C}_j})  \frac{{\bf F}^{(k-1)}_j}{\sqrt{d_i d_j}};\quad  k=1,\dots,
\end{align}
with ${\bf F}^{(0)}_i ={\bf S}_i$, which completes the proof.
\end{proof}
\end{theorem}
Following the iterative solution in Eq.~\eqref{eq:iterative_simplified}, we model the aggregation operation (for node $i$) for \methodadaptive as:
\begin{equation}
    \resizebox{0.88\linewidth}{!}{
    $    {\bf H}^{(k)}_i \leftarrow 2b_i{\bf X}'_i + b_i  \sum\limits_{v_j\in \tilde{\mathcal{N}}(v_i)} ({\mathcal{C}_i+\mathcal{C}_j})  \frac{{\bf H}^{(k-1)}_j}{\sqrt{d_i d_j}};\quad  k=1,\dots K, $}
    \label{eq:iter_gnn}
\end{equation}

where ${\bf X}'$ is the output of feature transformation, ${\bf H}^{(0)}={\bf X}'$, $K$ is the number gradient descent iterations, $\mathcal{C}_i$ can be considered as a positive scalar to control the level of ``local smoothness'' for node $i$, and $b_i$ can be calculated from $\{\mathcal{C}_j | j\in \tilde{\mathcal{N}}(i)\}$ as {\footnotesize{$b_i =  1/\left(2+ \sum\limits_{j\in \tilde{\mathcal{N}}(i)} (\mathcal{C}_i+\mathcal{C}_j)/{d_i}\right)$}}. However, in practice, $\mathcal{C}_i$ is usually unknown. One possible solution is to treat $\mathcal{C}_i$ as hyper-parameters.   But, treating $\mathcal{C}_i$ as hyper-parameters for all nodes is impractical, since there are in total $N$ of them and we do not have their prior knowledge. Thus, we instead parameterize $\mathcal{C}_i$ as a function of the information of the neighborhood of node $i$ as follows:
\begin{align}
    \mathcal{C}_i =  s\cdot \sigma \left( h_1\left(h_2\left( \left\{ {\bf X}'_j | j\in \tilde{\mathcal{N}}(i)  \right\}\right) \right) \right), \label{eq:cal_ci}
\end{align}
where $h_2(\cdot)$ is a function to transform the neighborhood information of node $i$ to a vector, while $h_1(\cdot)$ further transforms it to a scalar. $\sigma(\cdot)$ denotes the sigmoid function, which maps the output scalar from $h_1(\cdot)$ to $(0,1)$ and $s$ can be treated as a hyper-parameter controlling the upper bound of $\mathcal{C}_i$. $h_1(\cdot)$ can be modeled as a single layer fully-connected neural network. There are different designs for  $h_2(\cdot)$ such as channel-wise mean or variance~\citep{corso2020principal}. In this paper, we adopt channel-wise variance as the $h_2(\cdot)$ function (as a measure of diversity). 
APPNP can be regarded a special case of \methodadaptive, where {\tiny $\sigma \left( h_1\left(h_2\left( \left\{ {\bf X}'_j | j\in \tilde{\mathcal{N}}(i)  \right\}\right) \right) \right)$} produces a constant $1$ (i.e, $\mathcal{C}_i=s$) for all nodes. For the node classification task, the representation ${\bf H}^{(K)}$, which is obtained after $K$ iterations as in Eq.~\eqref{eq:iter_gnn}, is directly softmax normalized row-wisely and its $i$-th row  indicates the discrete class distribution of node $i$.

\section{Experimental Evaluation }\label{sec:experiments}
Although our contributions in this work are primarily towards mathematical understanding and unification of GNNs, in this section, we experimentally evaluate our proposed \methodadaptive  to demonstrate the promise of deriving new aggregations as solutions of denoising problems (not in striving for state-of-the-art GNN performance).  We begin with node classification experiments on standard graphs.  Next, we demonstrate the effectiveness of the proposed \methodadaptive in handling adaptive smoothness as manifested via adversarial attacks.  

\begin{figure*}[]%
     \centering
     \subfloat[\cora]{{\includegraphics[width=0.19\linewidth]{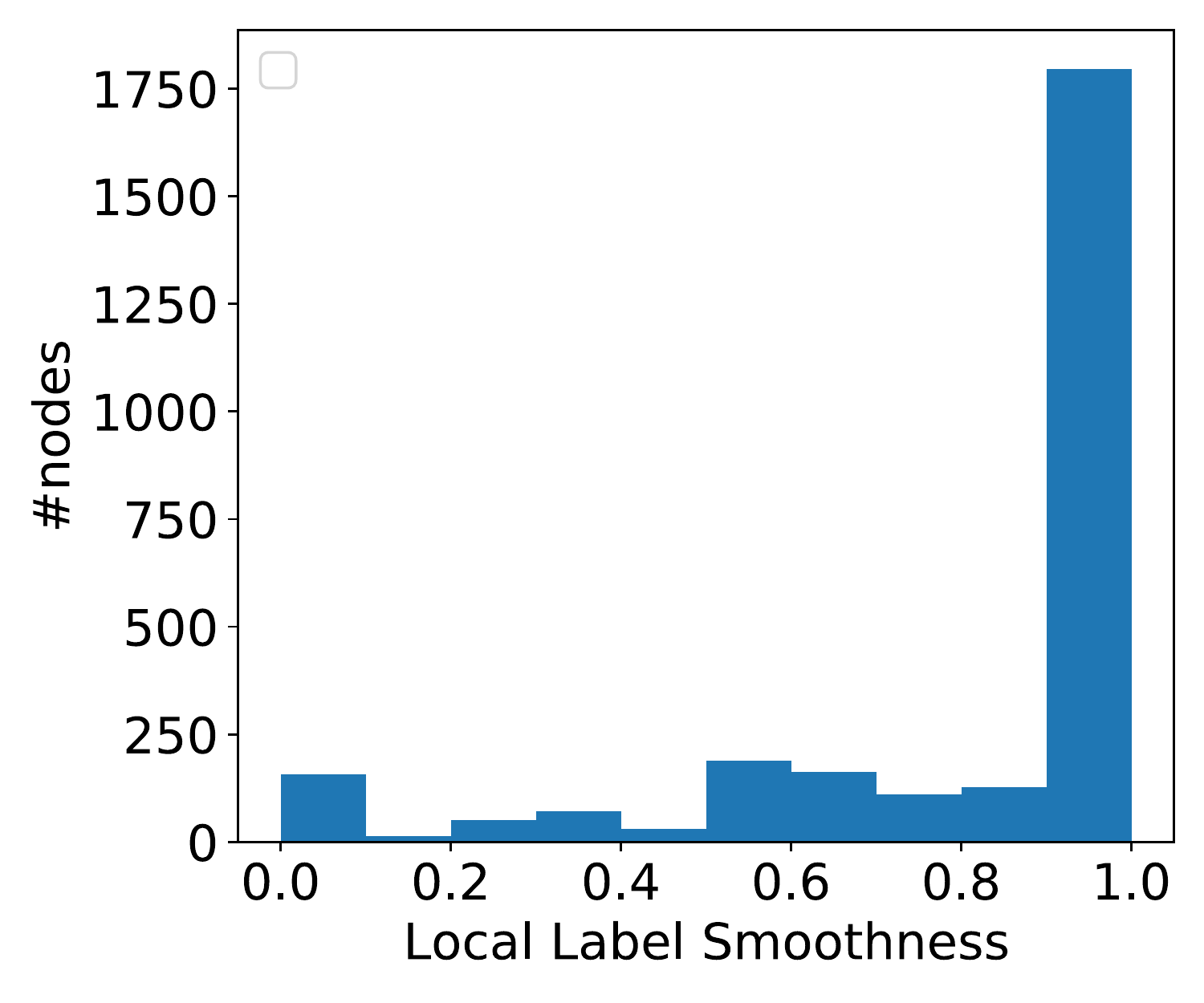} }}%
     \subfloat[\citeseer]{{\includegraphics[width=0.19\linewidth]{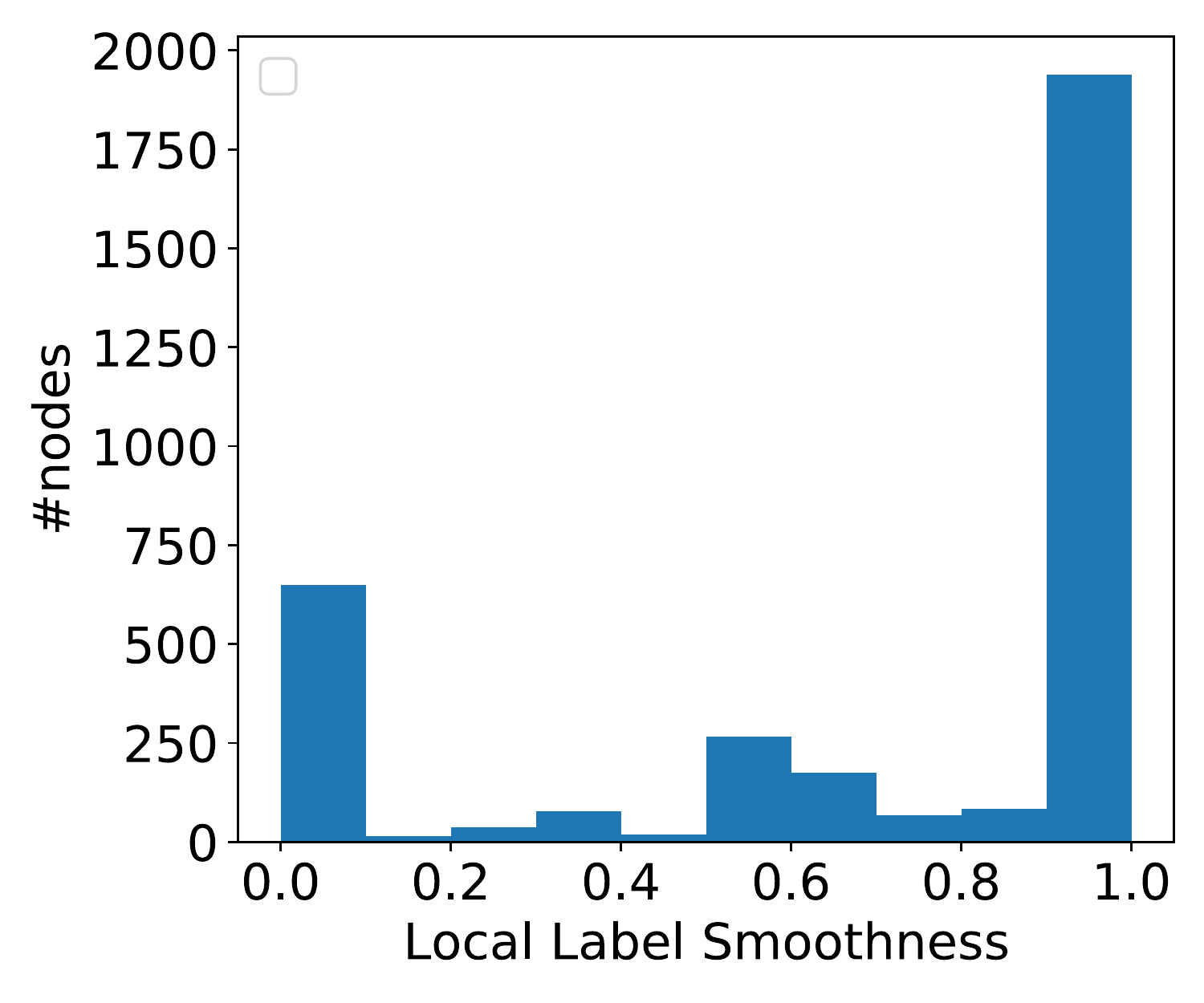} }}%
    \subfloat[\pubmed]{{\includegraphics[width=0.19\linewidth]{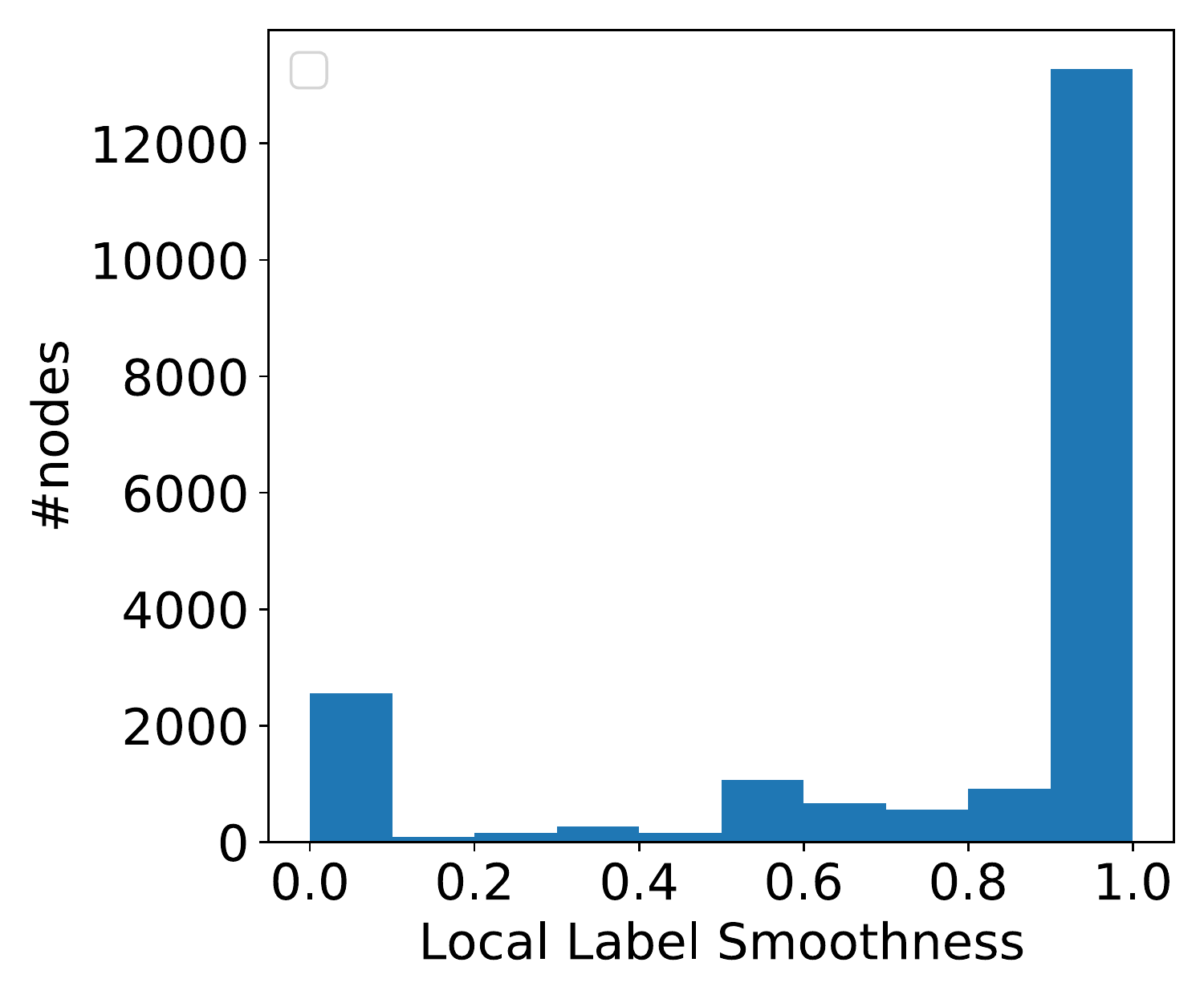} }}%
    \subfloat[\blogc]{{\includegraphics[width=0.19\linewidth]{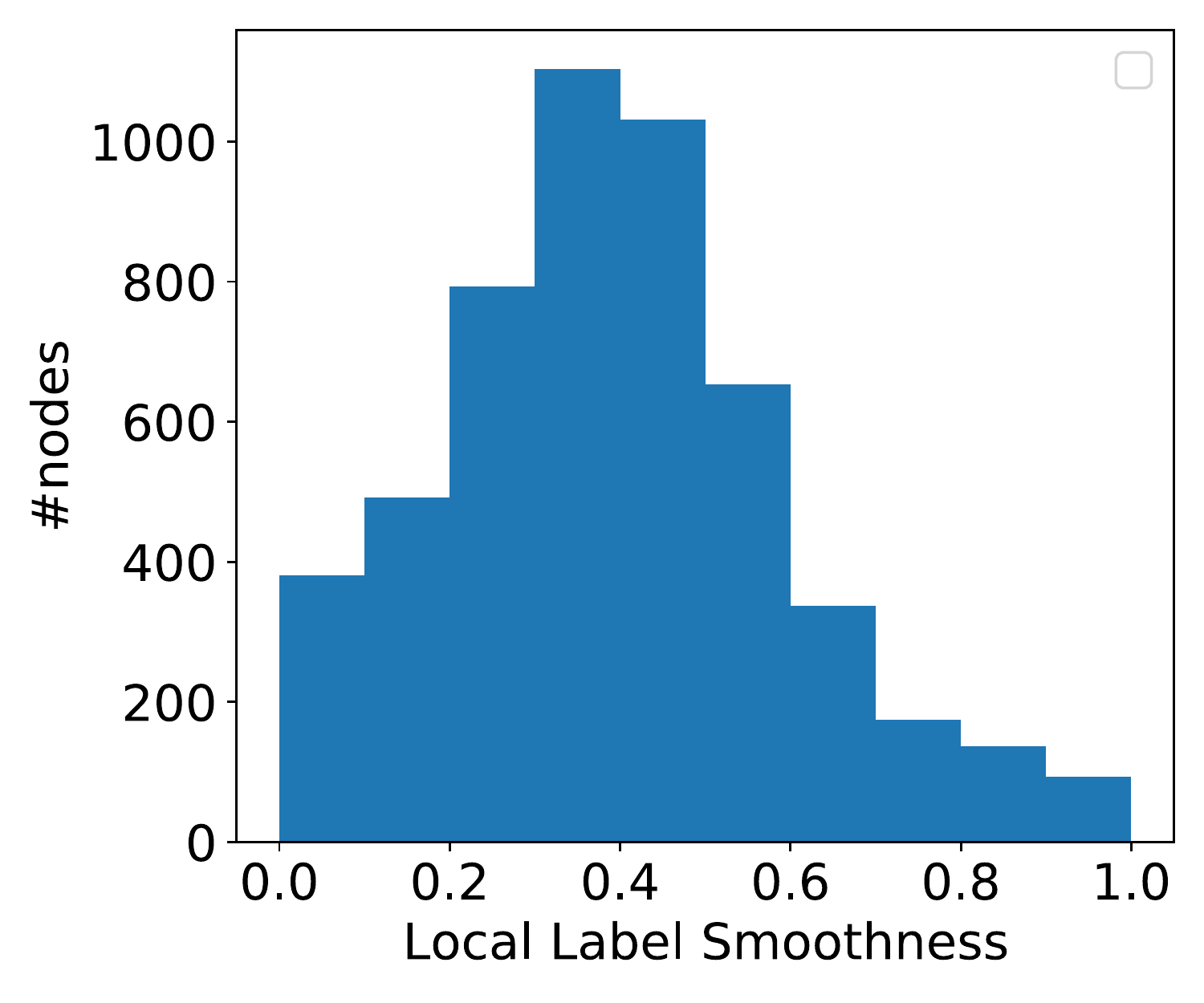} }}%
        \subfloat[\flickr]{{\includegraphics[width=0.19\linewidth]{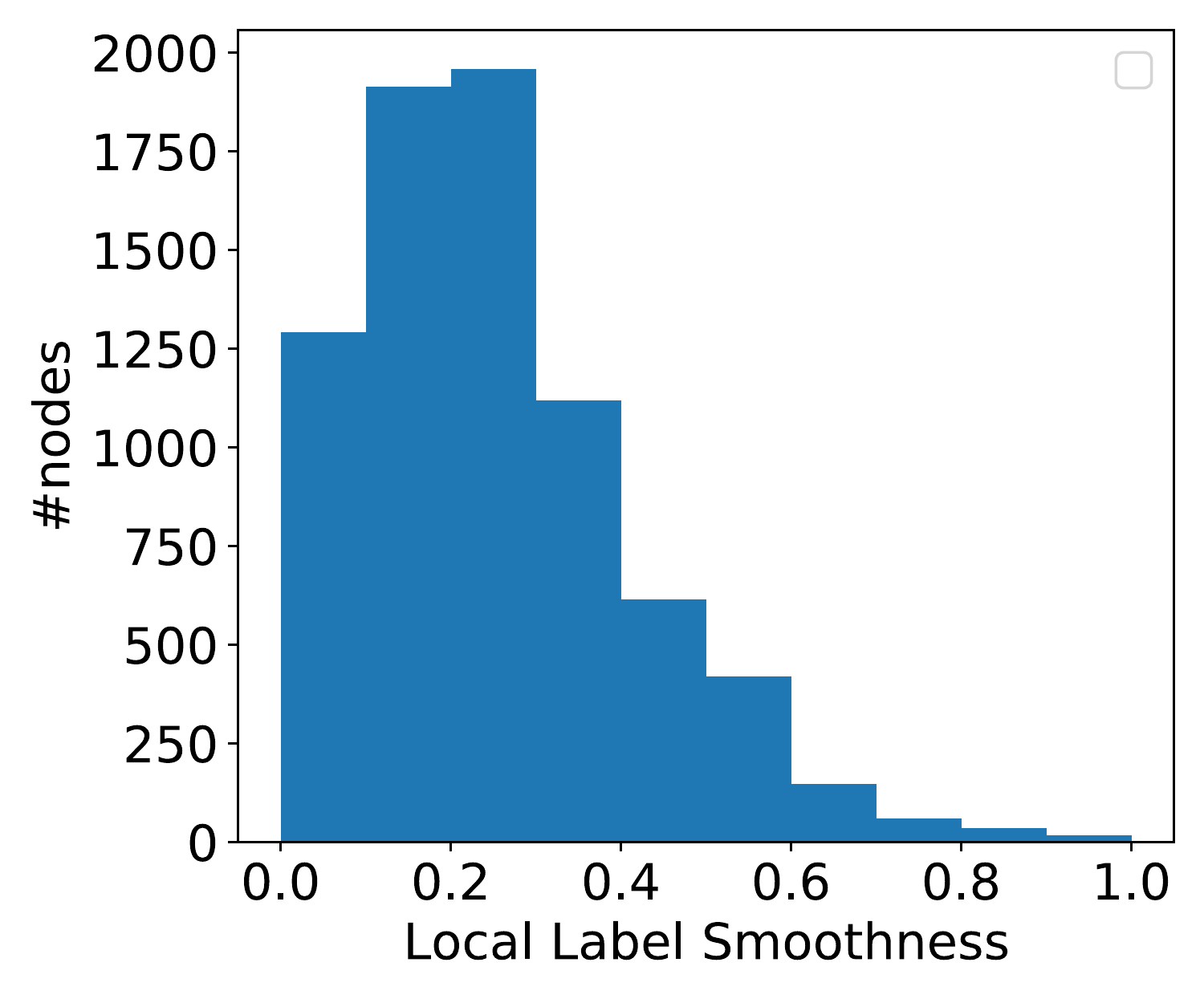} }\label{fig:flickr-alpha}}%

     \subfloat[\amazoncomp]{{\includegraphics[width=0.19\linewidth]{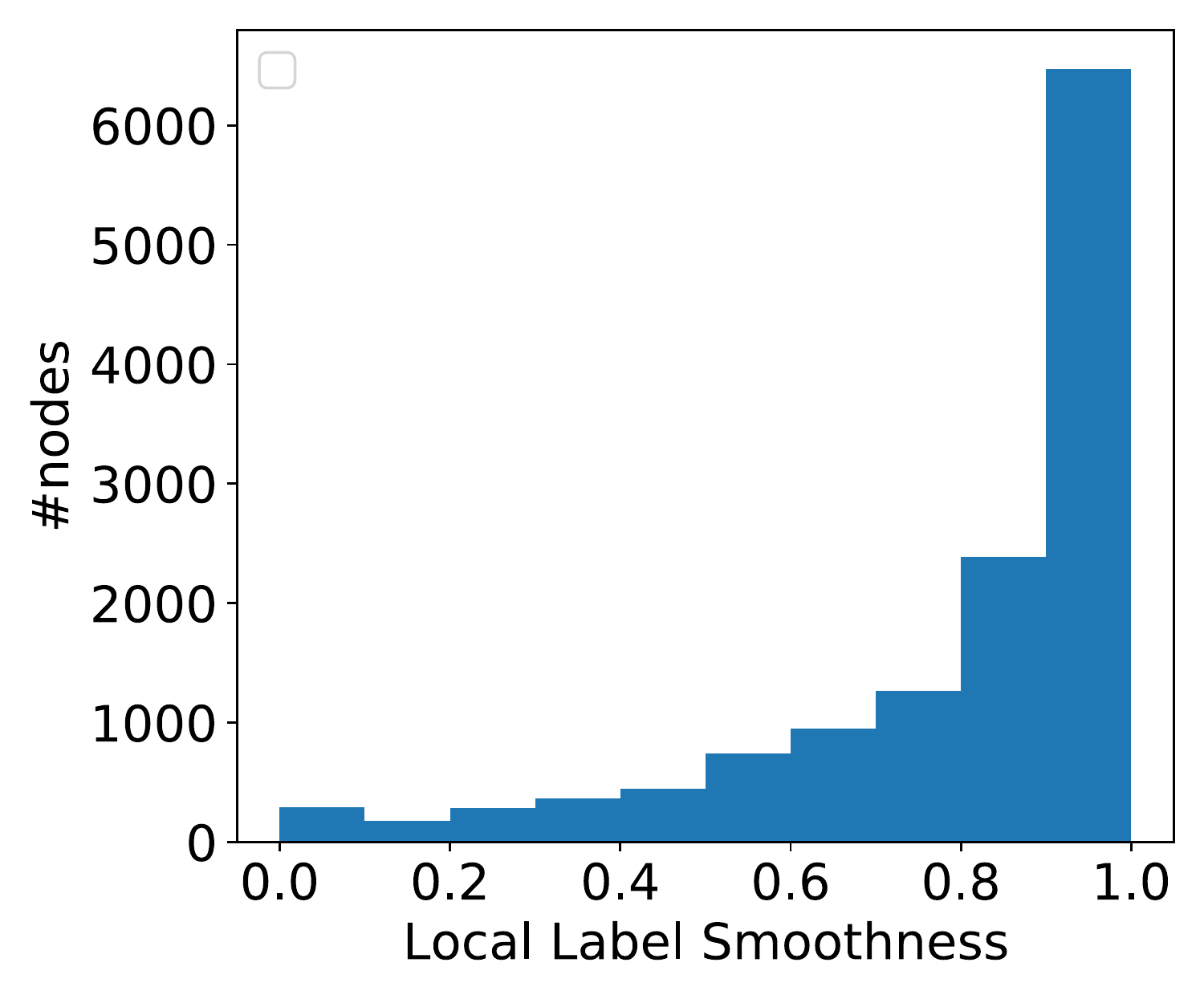} }}%
     \subfloat[\amazonphoto]{{\includegraphics[width=0.19\linewidth]{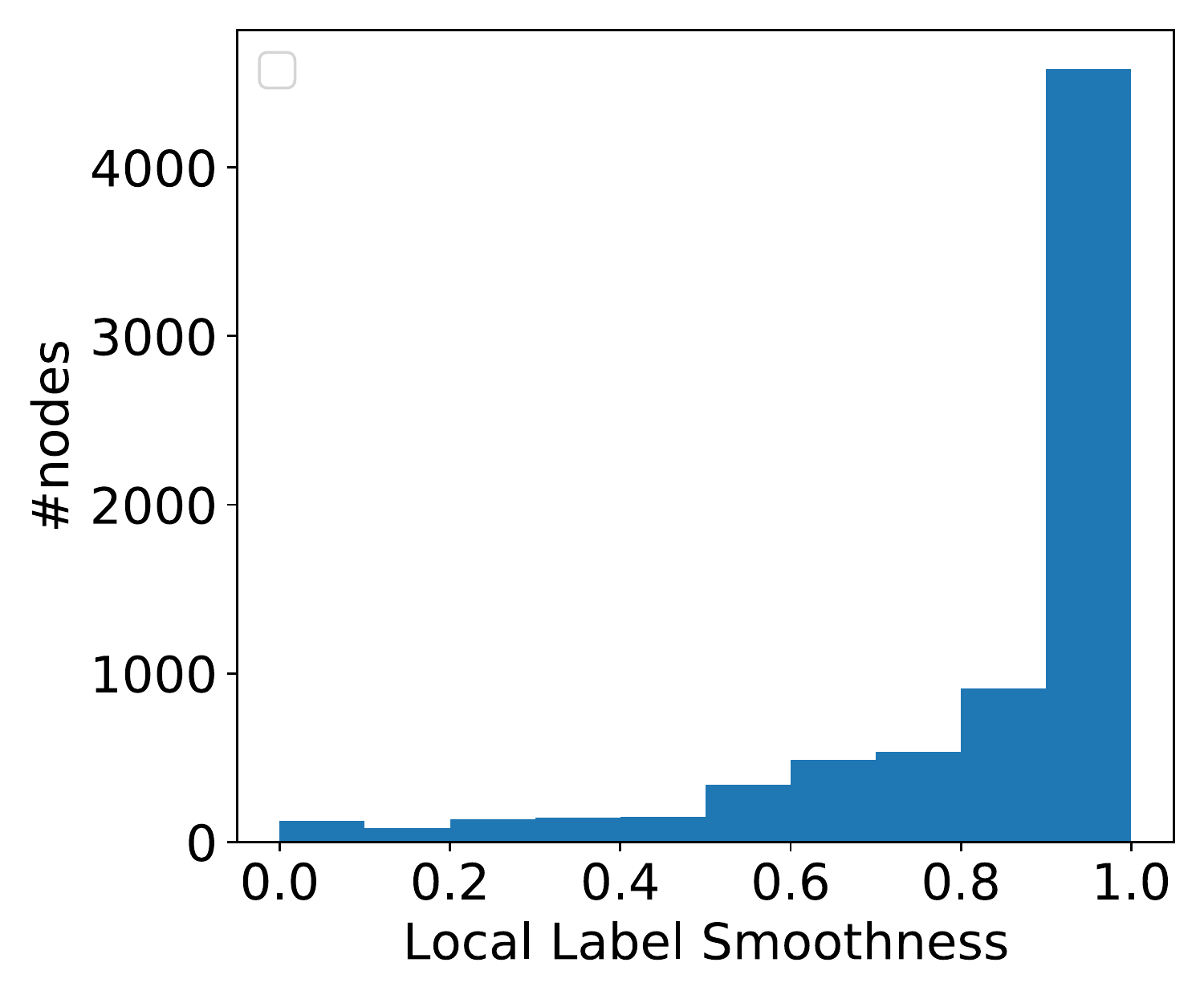} }}%
    \subfloat[\coauthorcs]{{\includegraphics[width=0.19\linewidth]{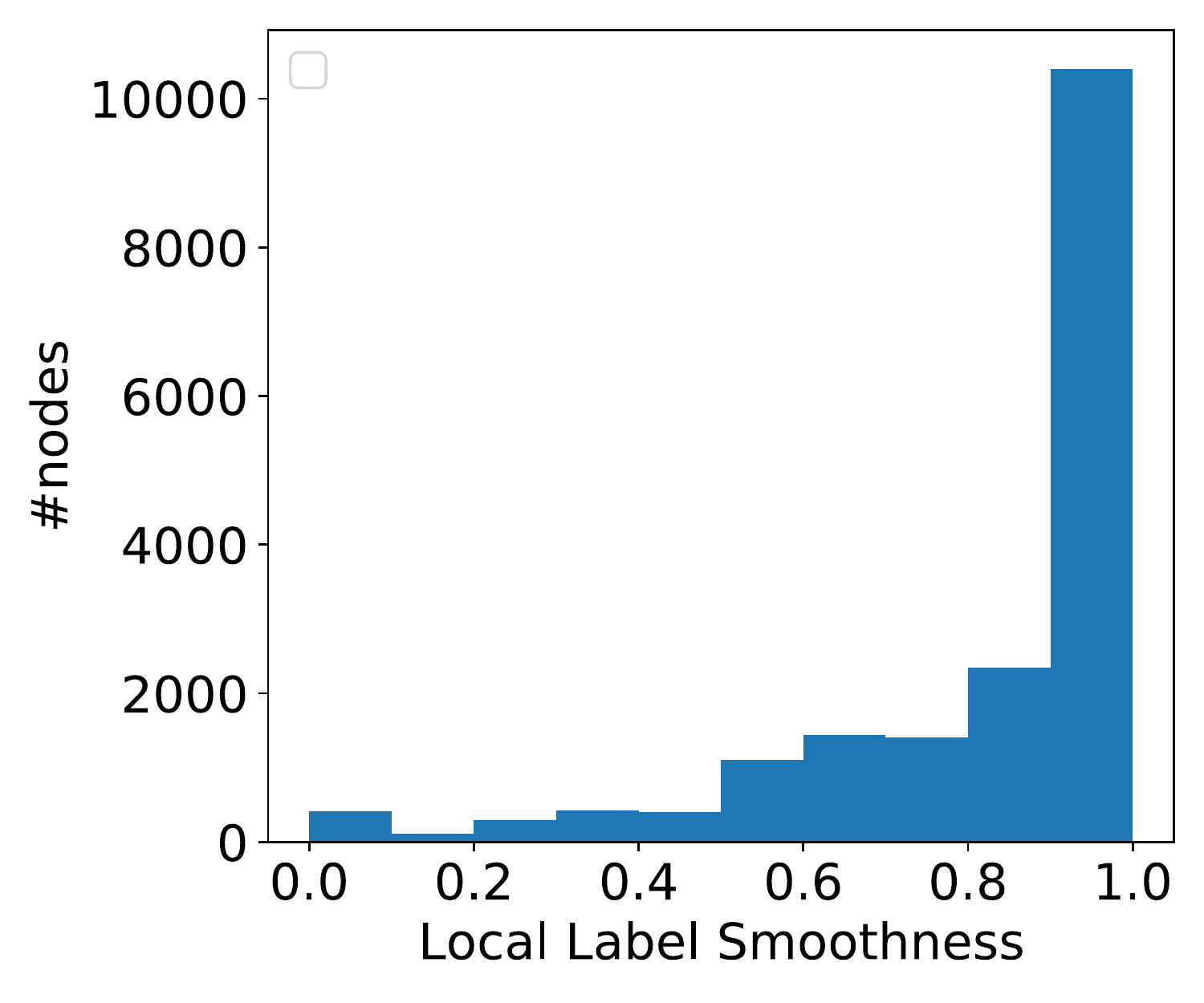} }}%
    \subfloat[\coauthorphys]{{\includegraphics[width=0.19\linewidth]{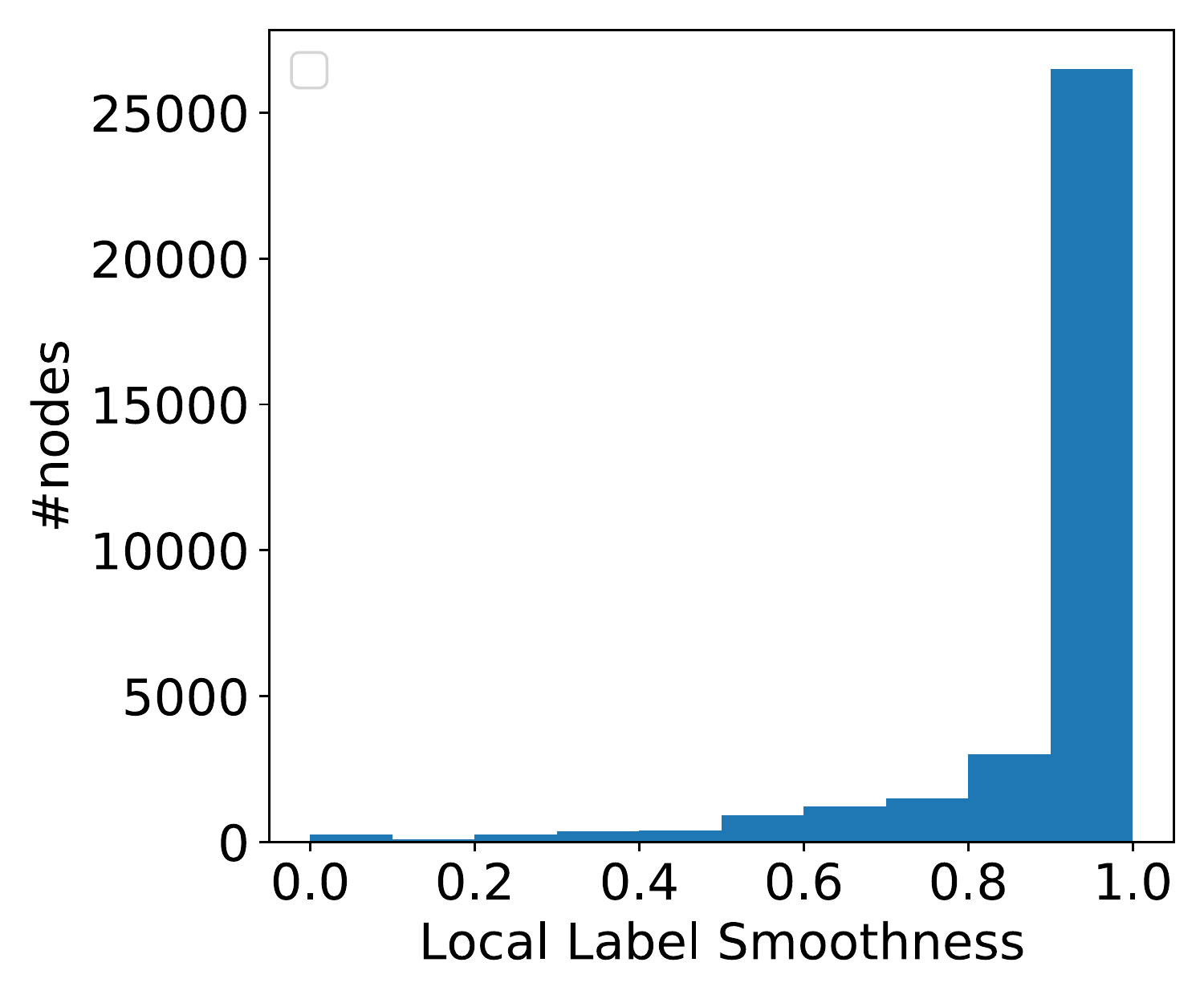} }}%
    \subfloat[\airusa]{{\includegraphics[width=0.19\linewidth]{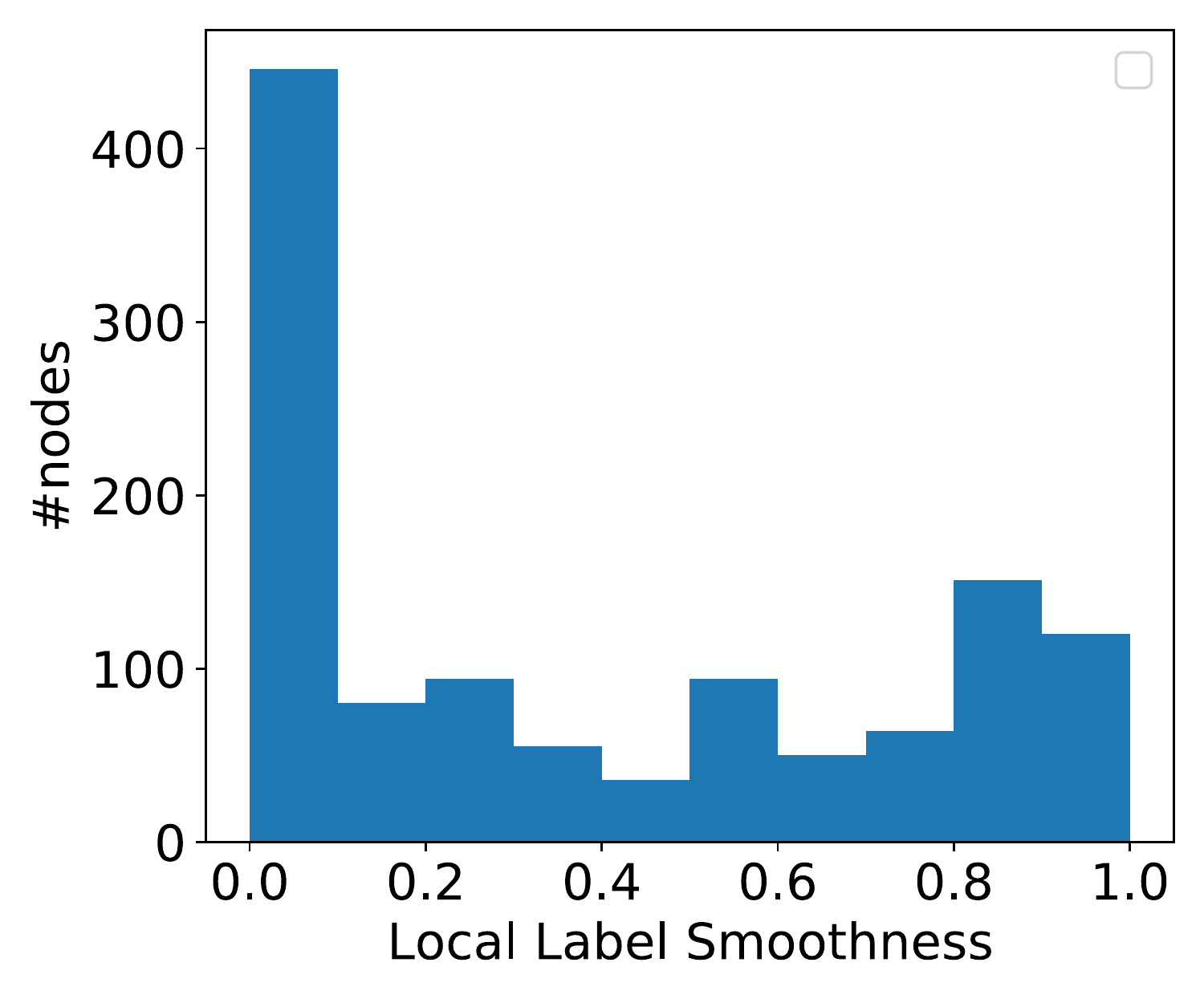} }}%
    \qquad
    \caption{Distribution of local label smoothness (homophily) on different graph datasets: note the non-homogeneity of smoothness values.} 
    \label{fig:localsmoothness}
\end{figure*}
\subsection{Node Classification}\label{sec:node_classification}

We first introduce the datasets and experimental settings in Section~\ref{sec:data_exp} and then present the results in~Section~\ref{sec:results}.
\subsubsection{Datasets and Experimental Settings}\label{sec:data_exp}
\textbf{Datasets.}  We use $10$ datasets from various domains including citation, social, co-authorship and co-purchase networks. Specifically, we use three citation networks including \cora, \citeseer, and \pubmed~\citep{sen2008collective}; two social networks including \blogc and \flickr~\citep{huang2017label}; two co-authorship networks including \coauthorcs and \coauthorphys~\citep{shchur2018pitfalls}; one transportation network, \airusa~\citep{wu2019net}; and two co-purchase networks including \amazoncomp and Amazon Photos~\citep{shchur2018pitfalls}. We provided detailed description of these datasets as follows.
\begin{compactitem}
    \item {\bf Citation Networks:} \cora, \citeseer and \pubmed are widely adopted benchmarks of GNN models. In these graphs, nodes represent documents and edges denote the citation links between them. Each node is associated bag-of-words features of its corresponding document and also a label indicating the research field of the document.
    \item {\bf Co-purchase Graph:} \amazoncomp and \amazonphoto are co-purchase graphs, where nodes represent items and edges indicate that two items are frequently bought
together. Each item is associated with bag-of-words features extract from its corresponding reviews. The labels of items are given by the category of them.
\item {\bf Co-authorship Graphs:} \coauthorcs and \coauthorphys are co-authorship graphs, where nodes are authors and edges indicating the co-authorship between authors. Each author is associated with some features representing the keywords of his/her papers. The label of an author indicates the his/her most active research field.
    \item {\bf Blogcatalog:} \blogc is an online blogging community where bloggers can follow each other. The \blogc graph consists of blogger as nodes while their social relations as edges. Each blogger is associated with some features generated from key words of his/her blogs. The bloggers are labeled according to their interests.
    \item {\bf Flickr:} \flickr is an image sharing platform. The \flickr graph consists users as its nodes and the following relation among users as its edges. The users are labeled with the groups they joined. 
\item {\bf \airusa:} \airusa is a air traffic graph, where each node is an airport in the US. Two nodes are considered as connected if there existing commercial flights between them. Nodes are labeled with the the passenger flow of each airport. 
\end{compactitem}
Some statistics about these datasets can be found in Table~\ref{tab:statistics}. To provide a sense of the local smoothness properties of these datasets, in addition to the summary statistics, we also illustrate the \emph{local label smoothness} distributions in Figure~\ref{fig:localsmoothness}: here, we define the local label smoothness of a node as the ratio of nodes in its neighborhood that share the same label with it. Specifically, for a node $v_i$ we formally define the local label smoothness as follows  
\begin{align}
    \text{ls}(i) = \frac{\sum\limits_{j\in \mathcal{N}(i)} \mathbf{1}\{ l(i) =l(j)\}}{|{\mathcal{N}}(i)|}\label{eq:local_label_smooth}
\end{align}
where $l(v_i)$ denotes the label of node $v_i$ and $\mathbf{1}\{a\}$ is an indicator function, which takes $1$ as output only when $a$ is true, otherwise $0$. Notably, as shown in Figure~\ref{fig:localsmoothness}, the large variety in local label smoothness within several real-world datasets including \blogc, \flickr and \airusa -- also observed in \citep{shah2020scale,pei2020geom,zhu2020beyond,zhu2020graph,chien2020adaptive,ma2021homophily} -- clearly motivates the importance of the adaptive smoothness assumption in \methodadaptive. 

\begin{table}[t!]
\small
\centering
\begin{tabular}{ccccc}
\toprule
                 & \#Nodes & \#Edges & \#Labels & \#Features  \\ \midrule
\cora             & 2708    & 13264   & 7        & 1433          \\ 
\citeseer         & 3327    & 12431   & 6        & 3703         \\ 
\pubmed           & 19717   & 108365  & 3        & 500          \\ 
\amazoncomp & 13381   & 504937  & 10       & 767           \\ 
\amazonphoto    & 7487    & 245573  & 8        & 745           \\ 
\coauthorcs      & 18333   & 182121  & 15       & 6805          \\ 
\coauthorphys  & 34493   & 530417  & 5        & 8415        \\ 
\blogc      & 5196    & 348682  & 6        & 8189          \\ 
\flickr  & 7575   & 487051  & 9      & 12047        \\ 
\airusa & 1190&28388 &4&238\\
\bottomrule
\end{tabular}
\caption{Dataset summary statistics.}
\label{tab:statistics}
\end{table}

\textbf{Experimental Settings.} 
For the citation networks, we use the standard split as provided in~\citet{kipf2016semi,yang2016revisiting}. For \blogc, \flickr and \airusa, we adopt the split provided in~\citet{zhao2020data}. For the citation networks, social networks and transportation network, we report  results averaged across 30 random seeds. For co-authorship and co-purchase networks, we utilize $20$ labels per class for training, $30$ nodes per class for validation and the remaining nodes for test. This process is repeated $20$ times, which results in $20$ different training/validation/test splits. For each split, the experiment is repeated for $20$ times with different initialization. The average results over $20\times 20$ experiments are reported. We compare our methods with the methods introduced in Section~\ref{sec:graph_and_gnn} including GCN, GAT and APPNP (we do not include PPNP due to scaling difficulty given the matrix inverse in Eq.~\eqref{eq:fa_ppnp}). For all methods, we tune the hyperparameters from the following options: 1) learning rate: $\{0.005, 0.01,0.05\}$; 2) weight decay  $\{5{e-}04,5{e-}05,5{e-}06,5{e-}07,5{e-}08\}$; and 3) dropout rate: $\{0.2,0.5,0.8\}$. For APPNP and our method \methodadaptive, we further tune the number of iterations $K$ and the upper bound $s$ for $c_i$ in Eq.~\eqref{eq:cal_ci} from the following range: 1) $K$: $\{2, 5,10\}$; and $s$: $\{1,9,19,29\}$. Note that we treat APPNP as a special case of our proposed method with $\mathcal{C}_i=s$ in Eq.~\eqref{eq:cal_ci}.

\begin{table}[t!]
\small
\centering
\caption{Node classification accuracy across datasets.}
\begin{adjustbox}{width=0.48\textwidth}
\begin{tabular}{lcccl}
\toprule
Accuracy (\%) & GCN & GAT & APPNP & \methodadaptive \\ \midrule
\cora            & 81.75$\pm$0.8 & 82.56$\pm$0.8 & 84.49$\pm$0.6 & 84.79$\pm$0.7* \\ 
\citeseer        & 70.13$\pm$1.0 & 70.77$\pm$0.8 & 71.97$\pm$0.6 & 72.17$\pm$0.6 \\ 
\pubmed          & 78.56$\pm$0.5 & 78.88$\pm$0.5 & 80.00$\pm$0.4 & 80.52$\pm$0.6*** \\ 
\amazoncomp & 82.79$\pm$1.3 & 83.01$\pm$1.5 & 82.99$\pm$1.6 & 83.40$\pm$1.3***\\
\amazonphoto    & 89.60$\pm$1.5 & 90.33$\pm$1.2 & 91.38$\pm$1.2 & 91.44$\pm$1.2\\
\coauthorcs     & 91.55$\pm$0.6 & 90.95$\pm$0.7 & 91.69$\pm$0.4 &  92.33$\pm$0.5***\\
\coauthorphys& 93.23$\pm$0.7 &  92.86$\pm$0.7    & 93.84$\pm$0.5 &  93.92$\pm$0.6\\
\blogc     & 71.38$\pm$2.7 & 72.90$\pm$1.2 & 92.43$\pm$0.9 & 93.33$\pm$0.3***\\
\flickr & 63.28$\pm$ 0.3& 52.17$\pm$1.0&83.19$\pm$0.4&  84.15$\pm$0.4*** \\
\airusa  &56.62$\pm$ 1.1 & 55.81$\pm$ 1.7  &56.20$\pm$1.2 & 57.32$\pm$1.2***\\
\bottomrule
\end{tabular}
\end{adjustbox}
\\
\scriptsize{$*$, ${*}{*}{*}$ indicate the improvement over APPNP is significant at $p < 0.1$ and $0.005$}
\vspace{-0.2in}
\label{tab:results}
\end{table}
\subsubsection{Performance Comparison}\label{sec:results}

We show results in Table~\ref{tab:results}, using two-sample $t$-test to evaluate significance.  We note a few main observations: 1) GAT outperforms GCN in most datasets, indicating that modeling adaptive local smoothness is generally helpful; 2) APPNP/\methodadaptive outperform GCN/GAT in most settings, suggesting iterative gradient descent offers advantages to single-step gradients, due to improved ability to achieve a denoising solution closer to the optimal; and 3) Most notably, our proposed \methodadaptive achieves consistently better performance than GCN/GAT, and outperforms or matches APPNP across datasets.  Outperformance of GAT suggests the importance of considering neighborhood information in learning local smoothness, while outperformance of APPNP suggests that adaptive local smoothness is advantaged versus fixed smoothness assumptions.

We note that APPNP is the closest contending method.  However, our reported averaged results consistently outperform it, and especially so on datasets where the the local label smoothness varies a lot across nodes, like \blogc, \flickr, and \airusa.  In addition to these datasets, \methodadaptive also achieves strongly significant improvements ($p < 0.005$) over APPNP on some datasets with lesser label smoothness diversity like \coauthorcs, \amazoncomp and \pubmed, and less significant improvements ($p < 0.1$) on \cora.  Comparatively, for datasets with extremely skewed local label smoothness distributions, where the majority of nodes have perfect, 1.0 label homophily (see Figure ~\ref{fig:localsmoothness}) like \amazonphoto, \coauthorphys, and \citeseer, improvement over APPNP is marginal.  APPNP shines in such cases, since its assumption of {\tiny$\sigma \left( h_1\left(h_2\left( \left\{ {\bf X}'_j | j\in \tilde{\mathcal{N}}(i)  \right\}\right) \right) \right)=1$} is ideal for these nodes (designating maximal local smoothness). Conversely, our model has the challenging task of learning $h_1(\cdot )$ and $h_2(\cdot)$ -- in such skewed cases, learning these functions may be relatively unfruitful, but still achieves strong performance. Overall, \methodadaptive can work well no matter whether the given graph has a skewed or diverse local label smoothness distribution, but especially shines when the local label smoothness is diverse.

\begin{table}[t!]
\small
\caption{Node classification accuracy, split across nodes with low/high local label smoothness.}
\begin{adjustbox}{width=0.48\textwidth}
\begin{tabular}{ccccc}
\hline
\multirow{2}{*}{Accuracy (\%)}     & \multicolumn{2}{c}{Low} & \multicolumn{2}{c}{High} \\ 
\cmidrule(lr){2-3} \cmidrule(lr){4-5}
           & APPNP     & ADA-UGNN     & APPNP      & ADA-UGNN     \\\bottomrule
\cora       & 38.40      & 40.17        & 90.65      & 90.76        \\
\citeseer   & 34.20      & 34.83        & 82.00       & 81.90         \\ 
\pubmed     & 41.08     & 44.44        & 88.32      & 88.33        \\ 
\amazoncomp & 44.08     & 45.81        & 88.08      & 88.31        \\ 
\amazonphoto  & 44.57     & 44.94        & 95.61      & 95.54        \\ 
\coauthorcs     & 43.94     & 51.79        & 96.31      & 96.26        \\ 
\coauthorphys      & 37.97     & 43.24        & 96.10      & 95.98        \\ 
\blogc       & 89.70      & 91.15        & 99.49      & 99.06        \\ 
\flickr     & 81.83     & 82.95        & 96.63      & 96.04        \\ 
\airusa     & 42.44     & 43.62        & 77.03      & 78.01        \\ \bottomrule
\end{tabular}
\end{adjustbox}
\label{tab:lh_acc}
\vspace{-0.2in}
\end{table}

\begin{figure*}[t!]%
     \centering
     \subfloat[0\%]{{\includegraphics[width=0.22\linewidth]{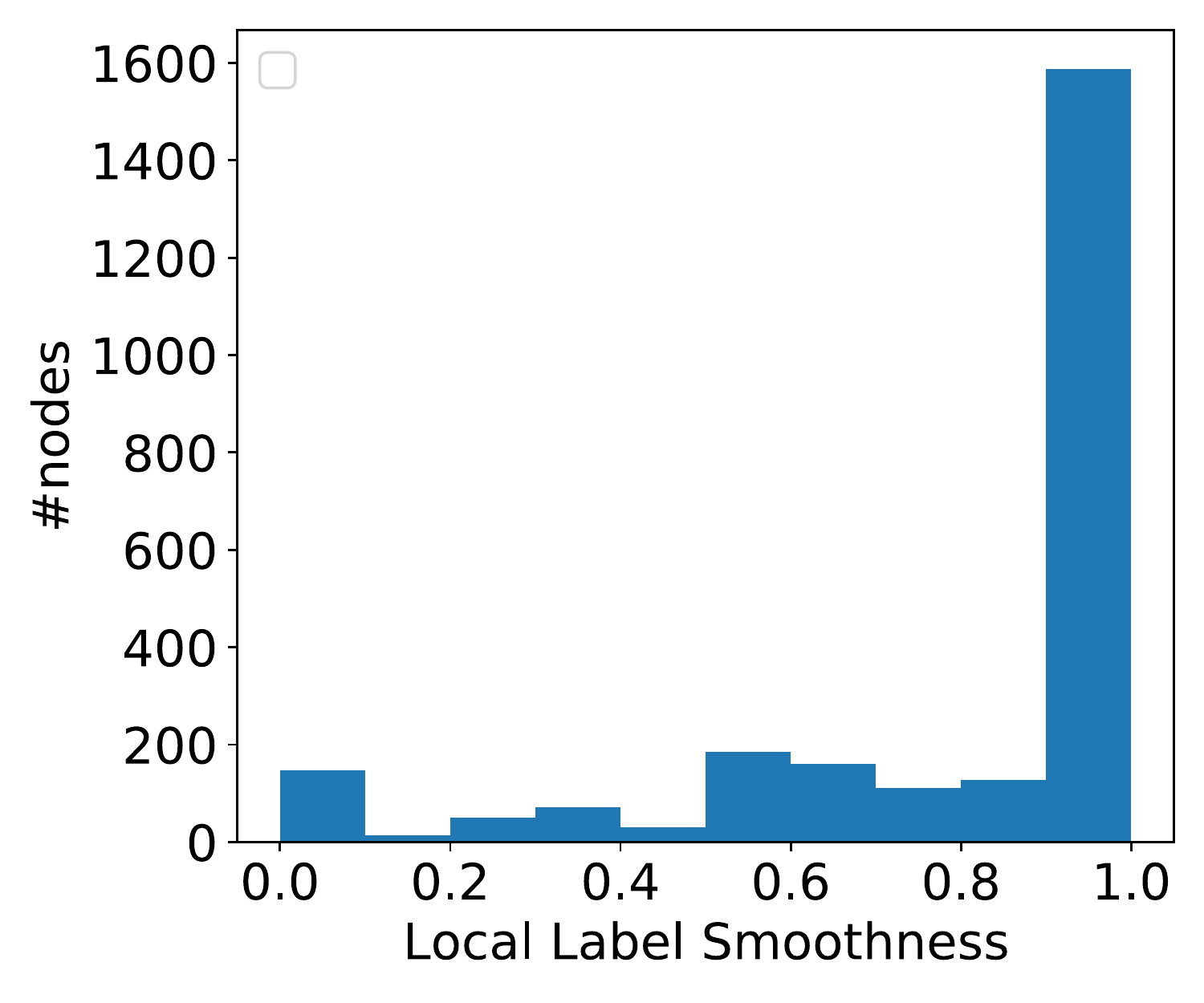}}}%
     \subfloat[5\%]{{\includegraphics[width=0.22\linewidth]{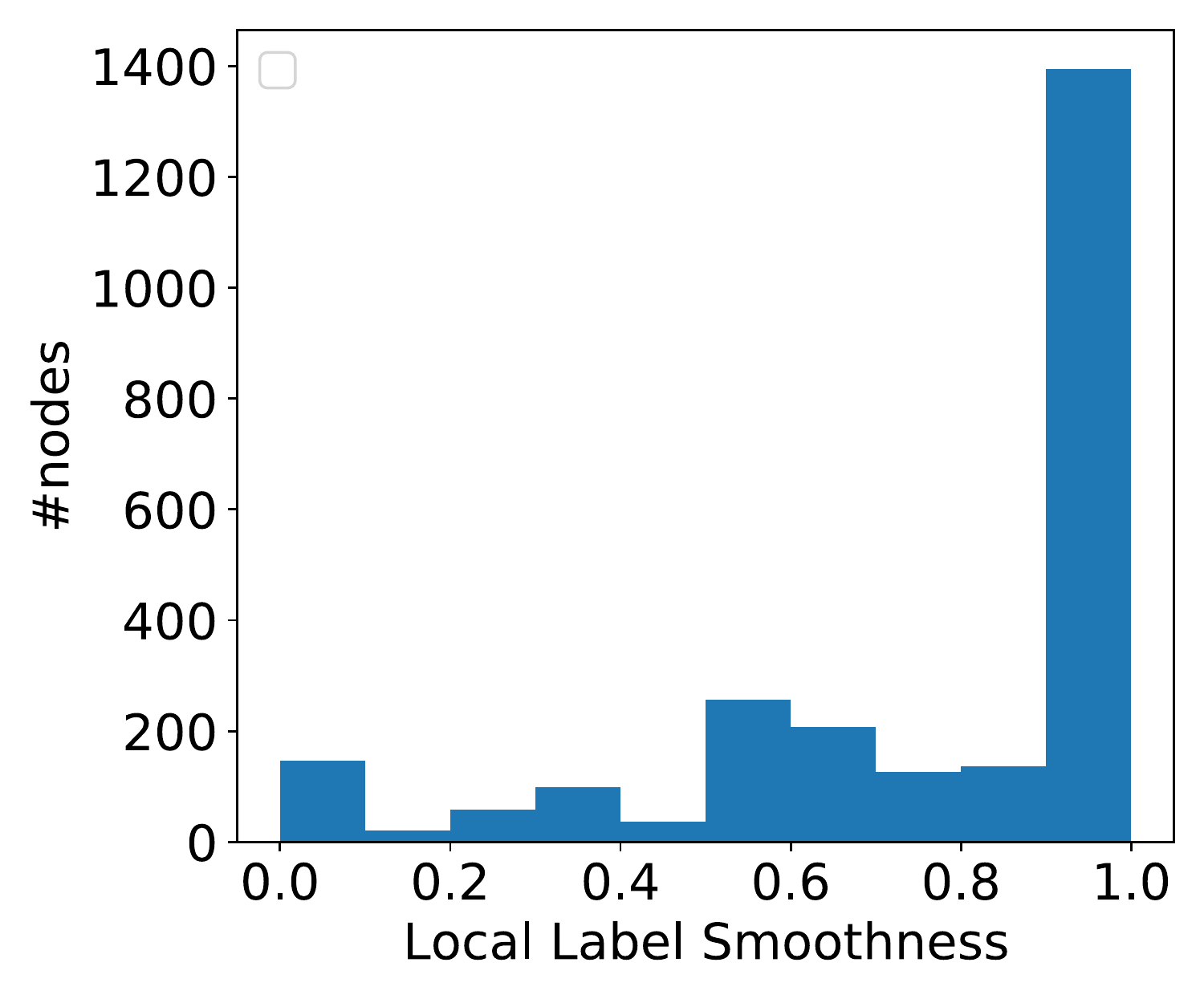} }}%
    \subfloat[15\%]{{\includegraphics[width=0.22\linewidth]{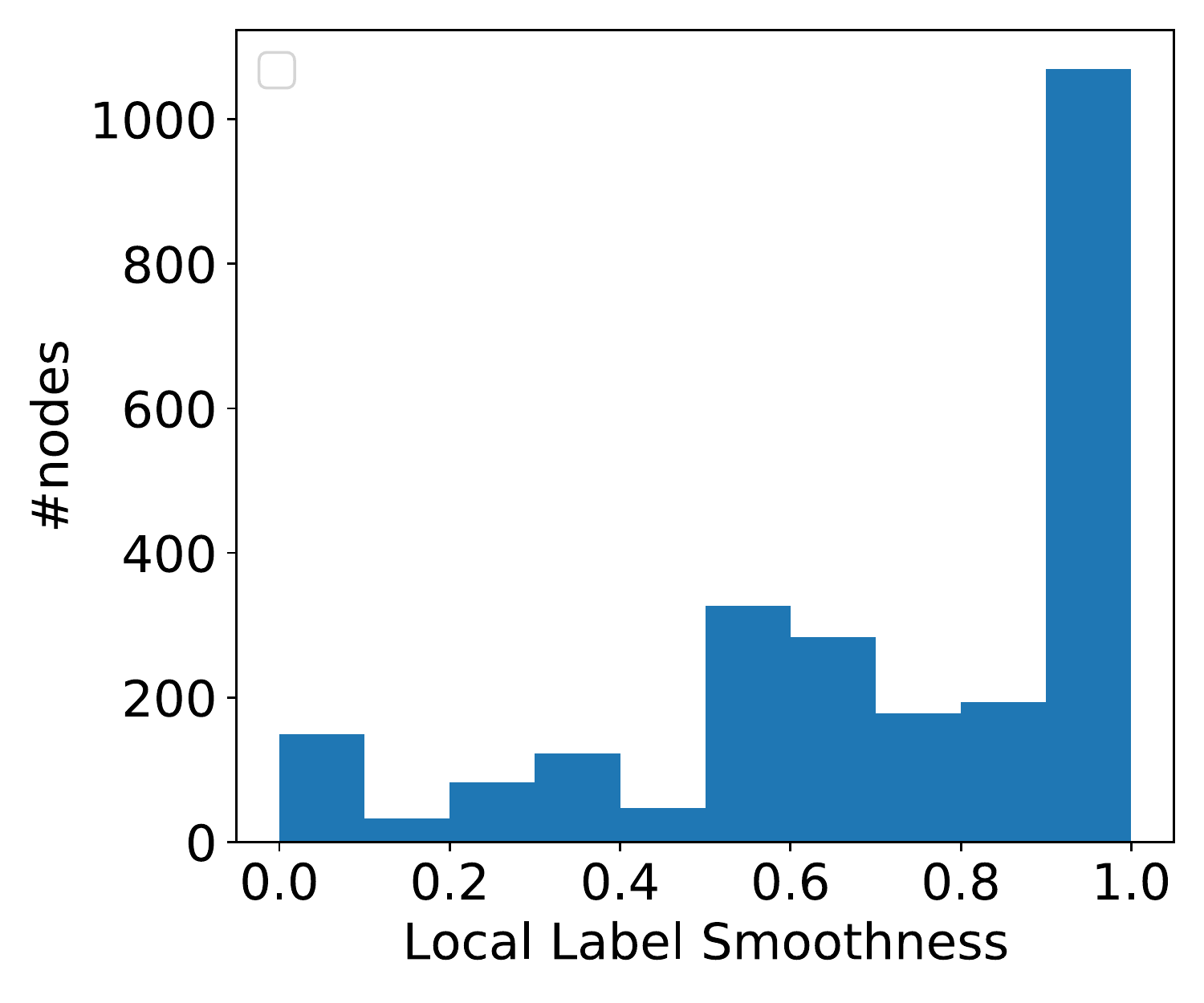} }}%
    \subfloat[25\%]{{\includegraphics[width=0.22\linewidth]{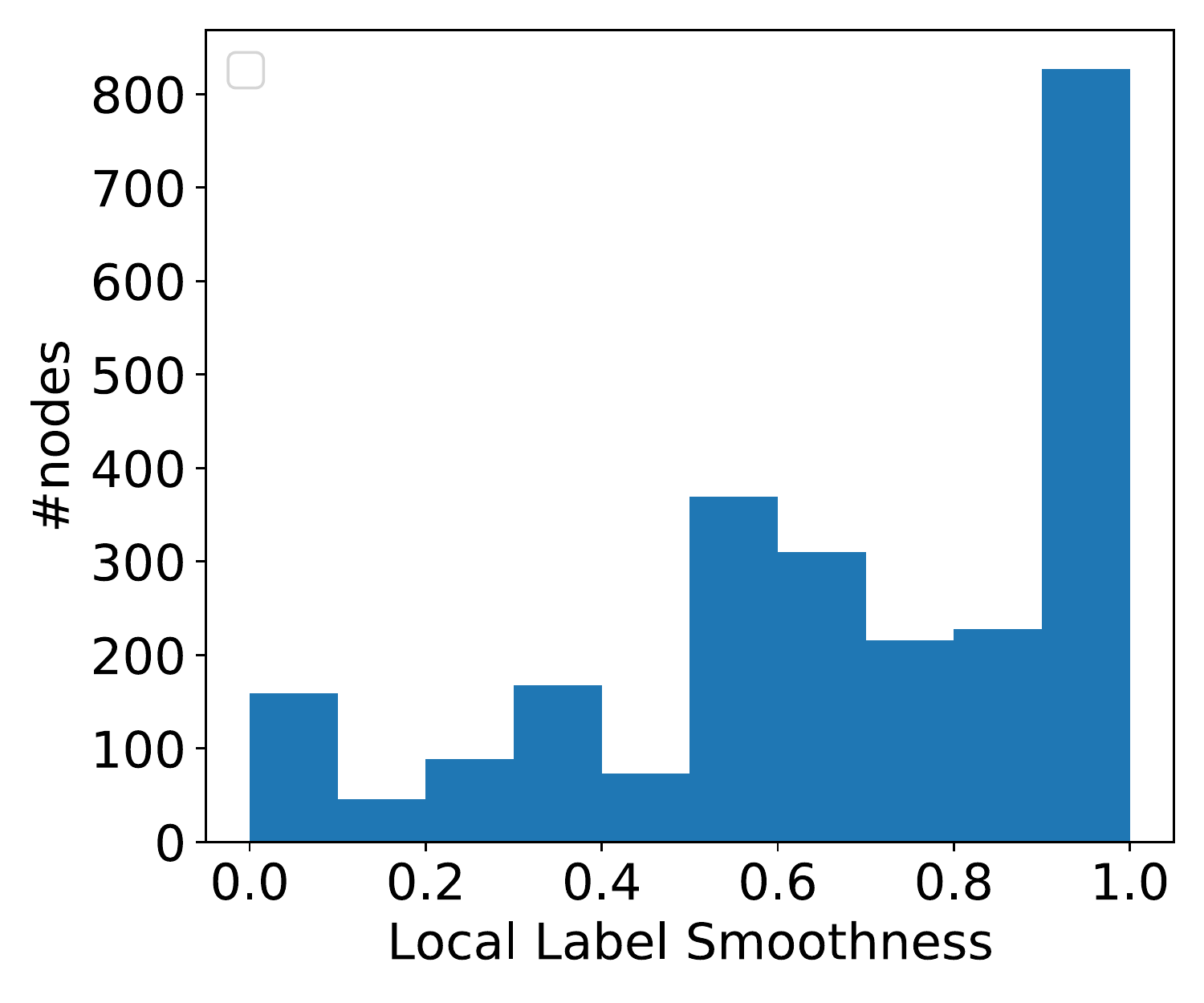} }}%
    \qquad
    \caption{Distribution of local label smoothness on \cora with various attack perturbation rates.} 
    \label{fig:cora-localsmoothness}
\end{figure*}

\begin{figure*}[ht!]%
     \centering
     \subfloat[\cora]{{\includegraphics[width=0.325\linewidth]{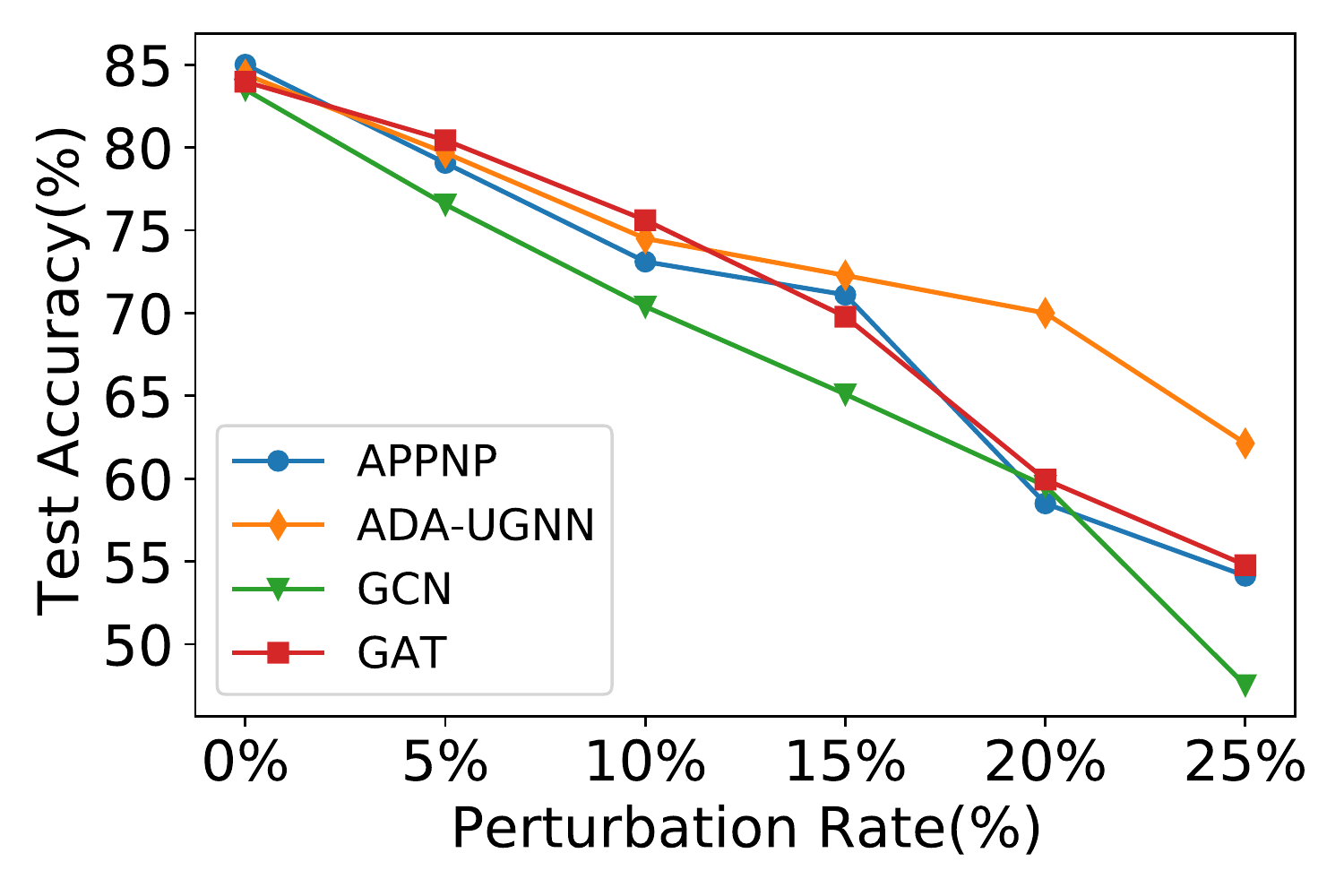} }}%
     \subfloat[\citeseer]{{\includegraphics[width=0.325\linewidth]{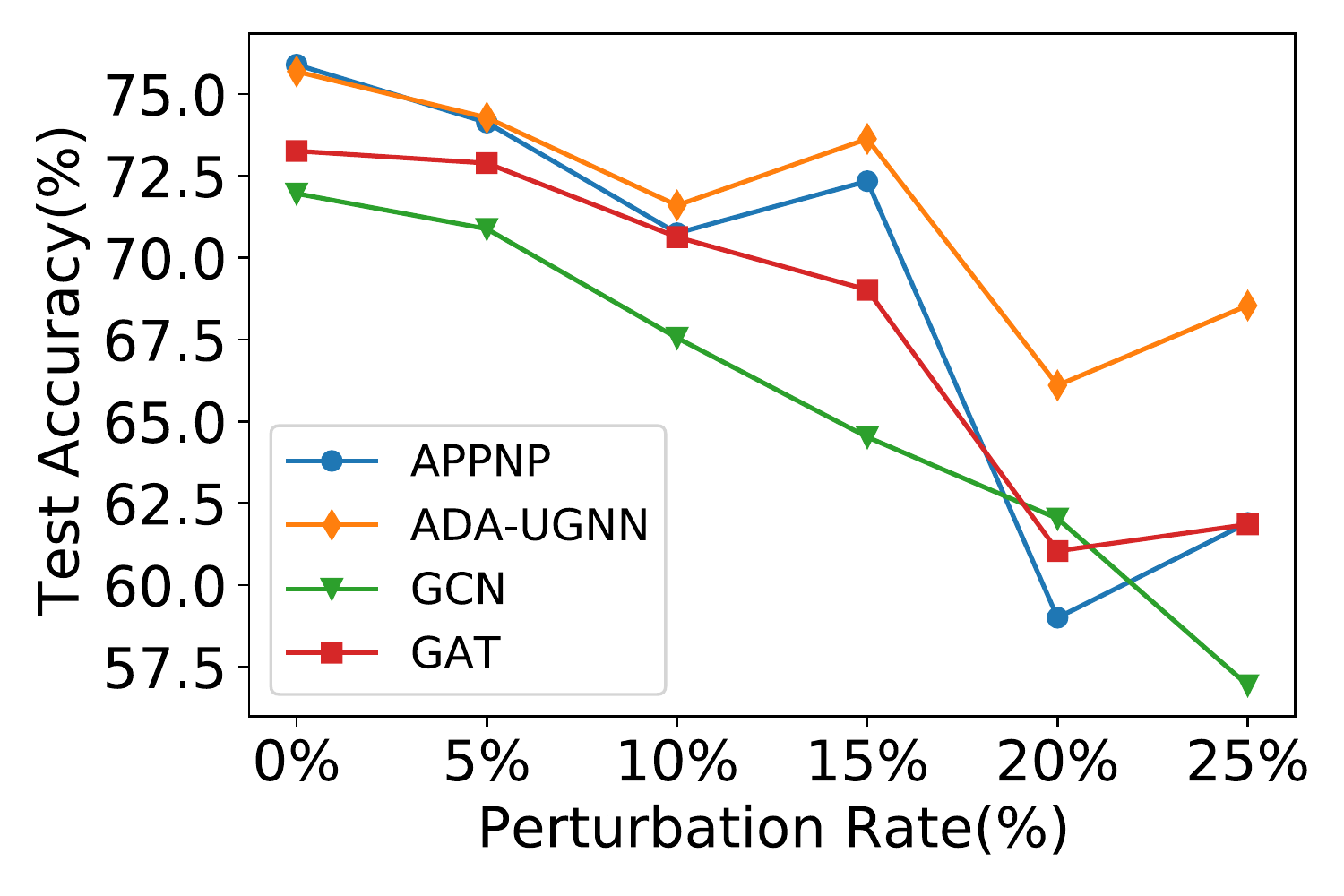} }}%
    \subfloat[\pubmed]{{\includegraphics[width=0.325\linewidth]{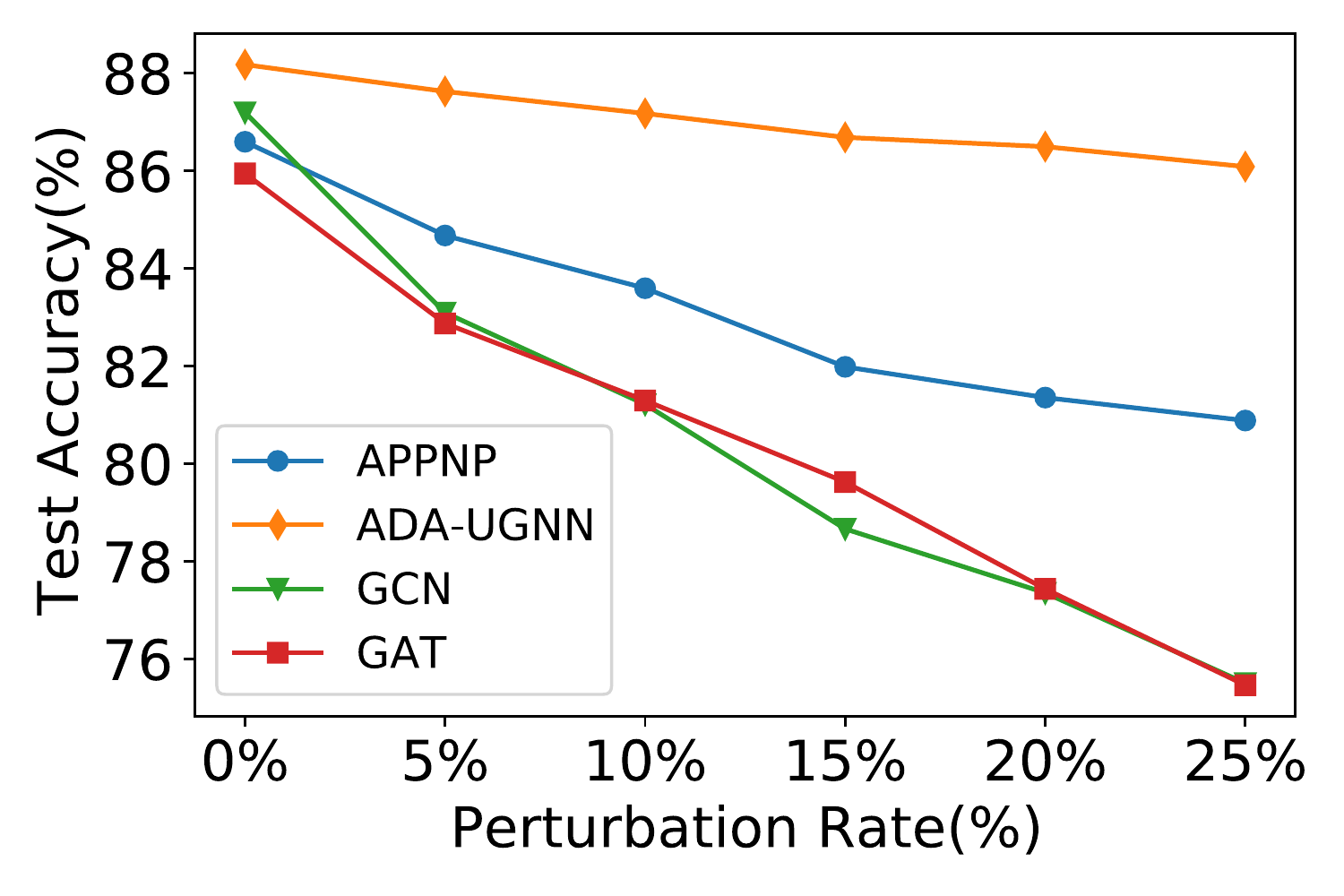} }}%
    \qquad
    \caption{Node classification accuracy under adversarial attacks. The proposed \methodadaptive maintains consistently strong performance even under high attack scale via its adaptive smoothness assumptions, where other methods degrade more quickly.} 
    \label{fig:robustness}
\vspace{-0.1in}
\end{figure*}

\subsubsection{Performance vs. Local Label Smoothness}

To further investigate how \methodadaptive works, we partition the nodes in the test set of each dataset into two groups: (1) \emph{high smoothness:} those with local label smoothness ${>}0.5$, and (2) \emph{low smoothness:} those with ${\leq}0.5$, and evaluate accuracy for APPNP and the proposed \methodadaptive for each group. The results for all datasets are shown in Table~\ref{tab:lh_acc}. 
Clearly, the proposed \methodadaptive consistently improves the performance for low-smoothness nodes in most datasets, while keeping comparable performance for high-smoothness nodes. Hence, for graphs where many nodes have low-level smoothness (like \blogc, \flickr or \airusa), \methodadaptive outperforms APPNP significantly in terms of overall performance. However, for graphs with very few low-smoothness nodes such as \coauthorphys, though \methodadaptive still significantly improves the performance over APPNP for those low smoothness nodes, the overall performance is similar to APPNP.

\subsection{Robustness Under Adversarial Attacks}\label{sec:adv_attack}

\begin{figure}[!h]%
\centering
\includegraphics[width=1\linewidth]{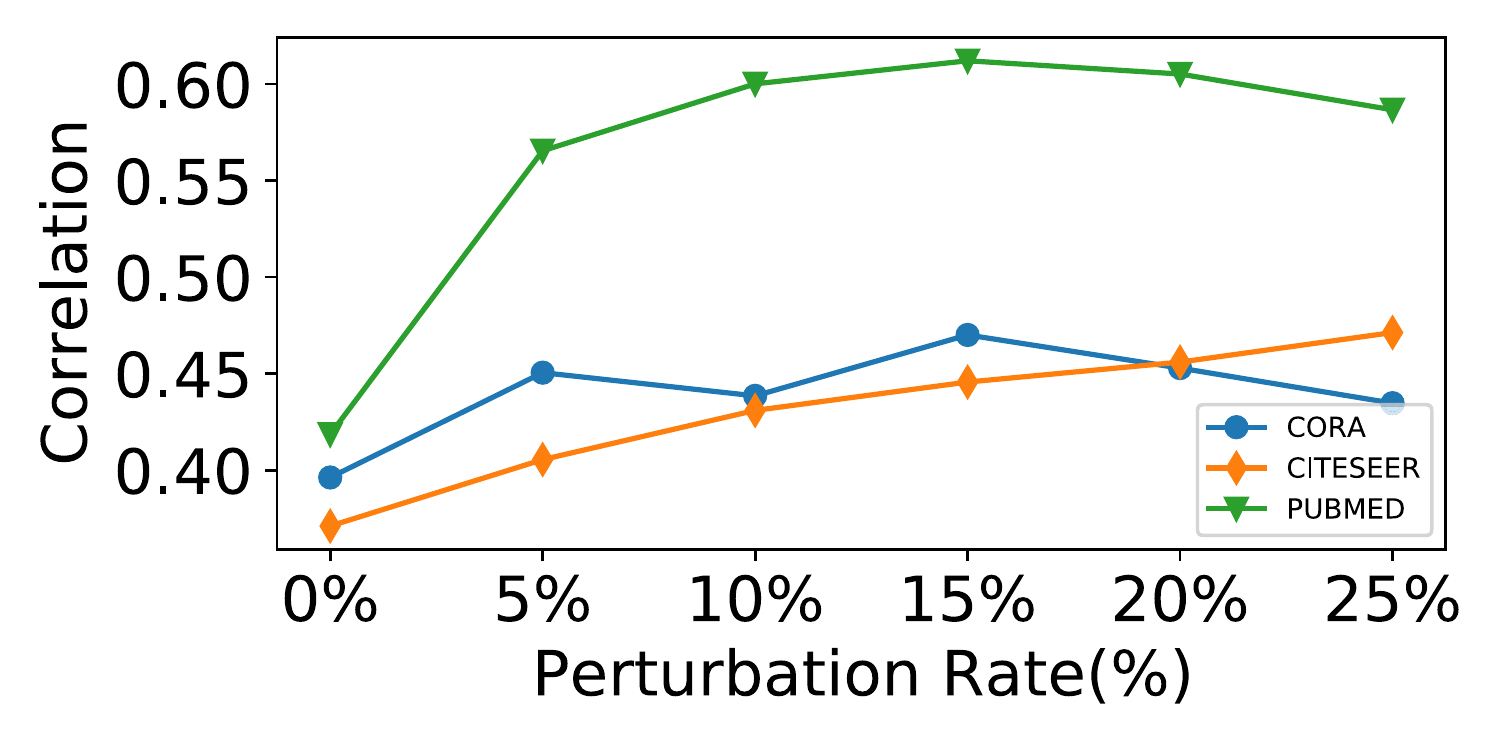}
\caption{Correlation between \methodadaptive's learned $\mathcal{C}_i$ scores and local label smoothness.}
\label{fig:c_corrrelation}
\end{figure}

Adversarial attacks on graphs tend to connect nodes from different classes and remove edges between nodes from the same class \citep{wu2019adversarial,jin2020graph}, producing graphs with varying local label smoothness after attack. To further demonstrate that \methodadaptive can handle graphs with varying local label smoothness better than alternatives, we conduct experiments to show its robustness under adversarial attacks. Specifically, we adopt Mettack~\citep{zugner2019adversarial} to perform the attacks. We utilize the attacked graphs (5\%-25\% perturb rate) from~\citet{jin2020graph} and follow the same setting, i.e. report average performance of each method over $10$ random seeds.
These attacked graphs are generated from \cora, \citeseer and \pubmed. As per prior work, we use only the largest connected component in each graph, and fix a $10/10/80$ training, validation and test split. Hence, the results in this section are not comparable with those in the previous section. We present the local smoothness distributions of the graphs generated by Mettack \citep{zugner2019adversarial} with different perturbation rate for \cora in Figure~\ref{fig:cora-localsmoothness}. The change in local smoothness distributions for \citeseer and \pubmed dataset are similar to \cora. We compare \methodadaptive with GCN, GAT and APPNP. Results under varying perturbation rates (attack intensities) are shown in Figure~\ref{fig:robustness}, with \methodadaptive in orange.  We have the following observations: 1) \methodadaptive is more stable than all three baselines, with the most graceful performance degradation under attack; and 2) \methodadaptive (orange) substantially outperforms APPNP by a large margin, especially in scenarios with high perturbation rate. 

These results further demonstrate that \methodadaptive can handle graphs with varying local label smoothness better than alternatives. Note that compared to the next-best contender (APPNP), \methodadaptive only introduces a constant number of additional parameters for modeling $h_1(\cdot)$ in Eq.~\eqref{eq:cal_ci}. Although \methodadaptive is not specifically designed to defend against adversarial attacks (and we do not claim it is the most suitable).

\subsubsection{Learning Smoothness under Attack}
 We investigate how \methodadaptive learns adaptive $\mathcal{C}_i$ under different attack perturbation ratios.
Ideally, for nodes with high local label smoothness, we expect the learned $\mathcal{C}_i$ to be larger, such that a higher-level local smoothness is enforced to this node during model training.
We consider the Pearson correlation between the learned $\mathcal{C}_i$ for all nodes with their local label smoothness (unknown during training). The correlation coefficients for the three datasets under various perturbation ratios are shown in Figure~\ref{fig:c_corrrelation}. In general, the learned $\mathcal{C}_i$ are strongly positively correlated with the local label smoothness under all settings on all three datasets.  
Moreover, compared with the clean graph ($0\%$ perturbation), the correlation scores are generally higher when the graphs are increasingly perturbed. This is likely because all the three datasets have highly skewed local label smoothness distributions as discussed in Section~\ref{sec:results}. Under perturbation, the label smoothness distributions of these three datasets become much more diverse (see Figure~\ref{fig:cora-localsmoothness} for a demonstration), which facilitates \methodadaptive to learn better $\mathcal{C}_i$. These findings are consistent with our original conjecture in Section~\ref{sec:results}. This also partially explains why \methodadaptive strongly outperforms APPNP under the attack setting on these datasets, compared to marginal outperformance under the clean graph setting.

\section{Related Works}\label{sec:related_works}
There are mainly two streams of work in designing GNN models, i.e, spectral-based and spatial-based. When designing spectral-based GNNs, graph convolution~\citep{shuman2013emerging}, defined based on spectral theory, is utilized to design GNN layers together with the feature transformation~\citep{bruna2013spectral,henaff2015deep,defferrard2016convolutional}. These spectral-based graph convolutions are tightly related with graph signal processing, and they can be regarded as graph filters. Low-pass graph filters can usually be adopted to denoise graph signals~\citep{chen2014signal}. In fact, most algorithms discussed in our work can be regarded as low-pass graph filters. With the emergence of GCN~\citep{kipf2016semi}, which can be regarded as a spectral-based and also a spatial-based graph convolution operator, numerous spatial-based GNN models have since been developed~\citep{hamilton2017inductive,velivckovic2017graph,monti2017geometric,gao2018large,gilmer2017neural}. A more comprehensive introduction on GNNs can be found at~\cite{ma2021deep}.  

Graph signal denoising aims to infer a clean graph signal given a noisy one, and can be usually formulated as a graph regularized optimization problem~\citep{chen2014signal}. Recently, several works connect GCN with graph signal denoising with Laplacian regularization~\citep{nt2019revisiting,zhao2019pairnorm}, finding the aggregation process in GCN models can be regarded as a first-order approximation of the optimal solution.  On the other hand, GNNs are also utilized to develop novel algorithms for graph denoising~\citep{chen2020graph}.  Unlike these works, our paper details how a family of GNN models can be unified with a graph signal denoising perspective, and shows  its promise for new architecture design. 

We noticed that one concurrent work very recently released to arXiv \cite{zhu2021interpreting}, which finds optimization commonalities between some GNN models. We approach our unified framework with signal denoising, which provides a better explanation of the framework and offers a new perspective. Furthermore, our observation of the adaptive local smoothness allows us to unify GAT into our framework, and propose a new GNN model \methodadaptive. Also, there is another concurrent work~\cite{fu2020understanding} connecting GNNs with graph signal denoising problem. Compared with it, our work connects diverse other models including GAT, PPNP, APPNP, DropEdge and Pairnorm with graph signal denosing problem, via UGNN's \textit{regularization-focused} paradigm and provides a novel connection to GAT from a local label smoothness angle. 




\section{Conclusion}

In this paper, we show how various representative GNN models including GCN, PPNP, APPNP and GAT can be unified mathematically as natural instances of graph denoising problems. Specifically, the aggregation operations in these models can be regarded as exactly or approximately addressing such denoising problems. With these observations, we propose a general framework, \method, which enables the development of novel and flexible GNN models from the denoising perspective via regularizer design. As an example demonstrating the promise of this paradigm, we instantiate the \method framework with a regularizer addressing adaptive local smoothness across nodes, and proposed and evaluated a suitable new GNN model, \methodadaptive.

\section*{Acknowledgements}
This research is supported by the National Science Foundation (NSF) under grant numbers IIS1714741, CNS1815636, IIS1845081, IIS1907704, DRL2025244, IIS1928278, IIS1955285, IOS2107215,\\ IOS2035472, Army Research Office (ARO) under grant number W911NF-21-1-0198, and a grant from Snap Inc.

\balance
\bibliography{ugnn}
\bibliographystyle{ACM-Reference-Format}

\end{document}